\documentclass[11pt]{article}
\usepackage{macro_wth}

\newcommand{\relu}{\sigma} %
\newcommand{\att}{\textsf{Att}} %
\newcommand{\TF}{\textsf{TF}}
\newcommand{\id}{\mathbb{I}}
\newcommand{\fornorm}[1]{\norm{#1}_{F}}

\newcommand{\commentout}[1]{}
\DeclarePairedDelimiter{\braces}{\{}{\}}		%
\DeclarePairedDelimiter{\bracks}{[}{]}		%
\DeclarePairedDelimiter{\parens}{(}{)}		%

\title{\bf How Well Can Transformers Emulate In-context Newton's Method?}

\author{Angeliki Giannou\thanks{ University of Wisconsin-Madison. Email: \texttt{giannou@wisc.edu}.}
\and Liu Yang\thanks{University of Wisconsin-Madison. Email: \texttt{liu.yang@wisc.edu}.}
\and Tianhao Wang\thanks{Yale University. Email: \texttt{tianhao.wang@yale.edu}.}
\and Dimitris Papailiopoulos\thanks{University of Wisconsin-Madison. Email: \texttt{dimitris@ece.wisc.edu}.}
\and Jason D. Lee\thanks{Princeton University. Email: \texttt{jasonlee@princeton.edu}.}
}

\begin{document}

\maketitle

\begin{abstract}

Transformer-based models have demonstrated remarkable in-context learning capabilities, prompting extensive research into its underlying mechanisms. Recent studies have suggested that Transformers can implement first-order optimization algorithms for in-context learning and even second order ones for the case of linear regression. In this work, we study whether Transformers can perform higher order optimization methods, beyond the case of linear regression. We establish that linear attention Transformers with ReLU layers can approximate second order optimization algorithms for the task of logistic regression and achieve $\epsilon$ error with only a logarithmic to the error more layers. As a by-product we demonstrate 
the ability of even linear attention-only Transformers in implementing a single step of Newton's iteration for matrix inversion with merely two layers. These results suggest the ability of the Transformer architecture to implement complex algorithms, beyond gradient descent.

\end{abstract}

\section{Introduction}\label{sec:intro}

Transformer networks have had a significant impact in machine learning, particularly in tasks related to natural language processing  and computer vision ~\citep{vaswani2017attention,khan2022transformers,yuan2021tokens,dosovitskiy2020image}. A key building block of Transformers is the self-attention mechanism, which enables the model to weigh the significance of different parts of the input data with respect to each other. This allows the model to capture long-range dependencies and learn complex patterns of the data, yielding state-of-the-art performance across a several tasks, including but not limited to language translation, text summarization, and conversational agents~\citep{vaswani2017attention,kenton2019bert}.

It has been long observed that Transformers are able to perform various downstream tasks at inference without any parameter updates \cite{brown2020language,lieber2021jurassic,black2022gpt}. This ability,  known as \emph{in-context learning}, has attracted the interest of the community, resulting in a line of works aiming to interpret and understand it.  Towards this direction, \citet{garg2022can} were the first to consider a setting, in which the ``language'' component of the problem is removed from the picture, allowing the authors to study the ability of Transformers to learn how to learn in regression settings. However, even for this simple setting, understanding the mechanics of the architecture that would allow such capability, is not yet very well understood. 

Following research has presented constructive methods to explain what type of algorithms these models can implement internally, by designing model weights that lead to a model that implements specific meta-algorithms.  
\citet{akyurek2022learning} constructed Transformers that implement one step of gradient descent with $O(1)$ layers. 
Other works have focused on the linear attention (removing the softmax) and have shown empirically \citep{von2022transformers} and theoretically \citep{ahn2023transformers,mahankali2023one} that the optimum for one layer, is in essence one step of preconditioned gradient descent. 
\citet{ahn2023transformers} also showed that the global minimizer in the two-layer case corresponds to gradient descent with adaptive step size, but the optimality is restricted to only a class of sparse weights configurations.
Beyond these, it still remains open how to characterize the global minimizer of the in-context loss for Transformers with multiple layers.

Another approach is to approximate the closed-form solution of linear regression, instead of minimizing the loss. 
For that purpose, one needs to be able to perform matrix inversion. 
There are multiple approaches to matrix inversion and in terms of iterative algorithms, one of the most popular ones is Newton's iteration \citep{SchulzInverse33}, which is a second-order method. Specifically, the method has a warm-up state with logarithmic, to the condition number of the matrix, steps and afterwards quadratic rate of convergence to arbitrary accuracy. 
The work of \citet{giannou2023looped}  showed that Transformers can implement second-order methods, like Newton's iteration\footnote{We use Newton's iteration for the matrix inversion algorithm and Newton's method for the optimization algorithm.}, for matrix inversion. %
{\citet{fu2023transformers} implemented Newton's iteration with Transformers for in-context linear regression; the authors also performed an empirical study to argue that Newton's iteration is closer to the trained models output rather than gradient descent.}

 Newton's iteration for matrix inversion is of the form $\bX_{t+1} = \bX_t(2\id - \bA\bX_t)$, where $\bA$ is the matrix we want to invert, and $\bX_t$ is the approximation of the inverse. %
The implementation of matrix inversion, serves as a stepping stone towards answering the following question: 
\begin{center}
``\textit{How well can Transformers implement higher-order optimization methods?}''
\end{center}
In pursuit of a concrete answer, we focus on the well-known Newton's method, a second-order optimization algorithm. %
For minimizing a function $f$, the method can be described as $\bx_{t+1} = \bx_t - \eta(\bx_t)(\nabla^2 f(\bx_t))^{-1}\nabla f(\bx_t)$ for some choice of step-size $\eta(\cdot)$.
To implement such updates, one essentially needs to first compute the step size $\eta(\bx_t)$, then invert the Hessian $\nabla^2 f(\bx_t)$, and finally multiply them together with the gradient $\nabla f(\bx_t)$.
In general, the step size $\eta(\bx_t)$ may also depend on quantities computed from $\nabla f(\bx_t)$ and $\nabla^2 f(\bx_t)$.
It is relatively straightforward for Transformers to perform operations like matrix transposition and multiplication, while the challenge is to devise an organic combination of all the above components to deliver an efficient implementation of Newton's method with convergence guarantees. 
In particular, it further requires rigorous convergence analysis to verify the effectiveness of the construction in concrete examples.

\paragraph{Main contributions.}
In this work we tackle the challenge from the perspective of in-context learning for the tasks of linear  and logistic regression. We consider Transformers with linear attention layers and position-wise feed-forward layers with the ReLU activation. 
We provide concrete constructions of Transformers to solve the two tasks, and derive explicit upper bounds on the depth and width of the model with respect to the targeted error threshold.
At a high level, our main findings are summarized in the following informal theorem.

\begin{theorem}[Informal]
Transformers can efficiently perform matrix inversion via Newton's iteration, based on which they can further 1) compute the least-square solution for linear regression, and 2) perform Newton's method to efficiently optimize the regularized logistic loss for logistic regression.
In particular, in the latter case, only  $\log\log(1/\epsilon)$ many layers  and $1/\epsilon^8$ width are required for the Transformer to implement Newton's method on the regularized logistic loss to achieve $\epsilon$ error.
\end{theorem}
\begin{wrapfigure}{r}{0.35\textwidth}
     \centering
\vspace{-1em}
\includegraphics[width = 0.3\textwidth]{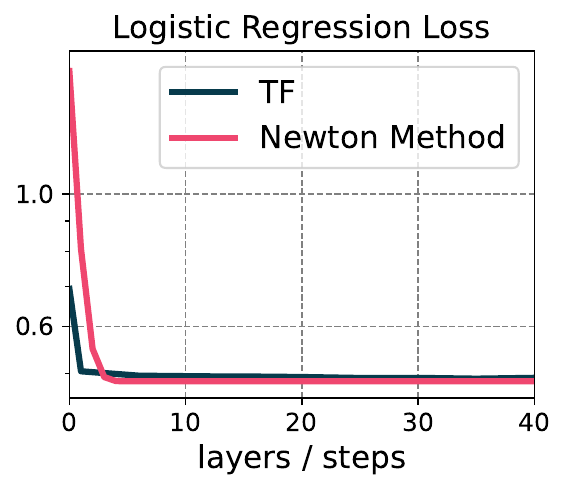}
\caption{The logistic regression loss of a trained Transformer with 1-40 layers and corresponding steps of Newton's method.}
\vspace{-1em}

\end{wrapfigure}

We also corroborate our results with experimental evidence. Interestingly, trained Transformers seem to outperform Newton's method  for the initial layers/steps. To understand what the models are actually learning we also train models with different number of layers for the simpler task of linear regression. We observe a significant difference in performance for models trained with or without layer norm. We compare the Transformers with variations of Newton's iteration with different order of convergence.  %

\section{Related work}\label{sec:related}
It has been observed that Transformer-based models have the ability of performing
in-context learning, as well as the ability of algorithmic reasoning \citep{brown2020language,nye2021show,wei2022emergent,wei2022chain,dasgupta2022language,zhou2022teaching}. Recently, \citet{garg2022can} initiated the mathematical formulation of the in-context learning problem, studying empirically the controlled setting of linear regression. Transformers were able to perform in-context linear regression, given only $(\bx_i,y_i)$ pairs, generated as $y_i = \bw_*^\top\bx_i$, which were not seen during training.
Later on, other works %
studied this setting both empirically and theoretically \citep{akyurek2022learning, von2022transformers,bai2023transformers, li2023transformers,guo2023transformers,chen2024training}.

Towards explaining this capability, 
\citet{akyurek2022learning,vonoswald2023transformers} showed by construction that Transformers can emulate gradient descent for the task of linear regression. \citet{vonoswald2023transformers}, also observed that empirically, one layer of linear attention Transformer had very similar performance with one step of gradient descent. 
Indeed, \citet{ahn2023transformers, mahankali2023one} proved that %
the global minimizer of the in-context learning loss for linear regression corresponds to one step of (preconditioned) gradient descent.

Related to our work is also the work of \citet{bai2023transformers}, which is the only work - to the best of our knowledge - that provides a construction of gradient based algorithms for various in-context learning tasks, including logistic regression; they also demonstrate the ability of Transformer based models to perform algorithm selection. In the case of learning non-linear functions, \citet{cheng2023transformers} showed that Transformers can perform functional gradient descent.

Focusing on performing linear regression through matrix inversion, the recent work  of \citet{vonoswald2023uncovering} is of interest. The authors approximate the inverse of a matrix using Neumman series.  For the task of linear regression this approach requires less memory compared to Newton's method, but it has a linear rate of convergence  \citep{wu2014large}. 

Considering higher order optimization methods, \citet{giannou2023looped} first implemented matrix inversion using Newton iteration, their construction though is sub-optimal, since it uses thirteen layers to perform just one step of the method and it is given in a general template. %
In a recent work  by  \cite{fu2023transformers}, the authors use  Newton's iteration for matrix inversion to approximate the closed form solution of linear regression; they compare it with a $12$-layer trained transformer, by linear-probing each layer's output and comparing it with steps of the iterative algorithm. They furthermore conclude that Transformers are closer to Newton's iteration rather than gradient descent.%

In terms of the optimization dynamics, \citet{zhang2023trained} proved  for one-layer linear attention the convergence of gradient flow to the global minimizer of the population loss given suitable initialization.
\citet{huang2023context} showed the convergence of gradient descent for training softmax-attention Transformer under certain orthogonality assumption on the data features.

\section{Preliminaries}\label{sec:prelim}
\paragraph{Notation.}
We use lowercase bold letters for vectors \eg $\bx,\by,$ and by convention we consider them to be column vectors; for matrices we use uppercase bold letters \eg $\bA,\bB$. 
We use $\lambda(\bA),\sigma(\bA)$ to denote the eigenvalues and singular values of a matrix $\bA$ respectively;  we use $\kappa(\bA)=\sigma_{\max}(\bA)/\sigma_{\min}(\bA)$ to denote the condition number of $\bA$. 
For a Positive Symmetric Definite (PSD) matrix $\bA$, we denote $\|\bx\|_\bA = \sqrt{\bx^\top\bA\bx}$. We use 
 $\zero_d$ and  $\id_d$ to denote a $d\times d$ matrix of zeros and the $d\times d$ identity matrix, respectively. If not specified otherwise, we use $*$ to denote inconsequential values.

\subsection{The Transformer architecture}
There are multiple variations of the Transformer architecture, depending on which type of attention (e.g., softmax, ReLU, and linear) is used.
In this work, we focus on Transformers with linear self-attention described as follows:
For each layer, let $\bH \in \RR^{d\times n}$ be the input, where each column is a $d$-dimensional embedding vector for each token.
Let $H$ be the number of attention heads, and for each head $i\in[H]$, we denote by $\bW_K^{(i)}, \bW_Q^{(i)}, \bW_V^{(i)} \in \RR^{d\times d}$ the key, query, and value weight matrices, respectively.
Further let $\bW_1\in\RR^{D\times d}$ and $\bW_2\in\RR^{d\times D}$ be the weights of the feed-forward network, then the output of this layer is given by computing consecutively,
\begin{subequations}
\label{eq:TF}
\begin{align}
    \att({\bH}) &= \bH + \sum_{i=1}^H \bW_V^{(i)}\bH(\bW_K^{(i)}\bH)^\top(\bW_Q^{(i)}\bH),\label{eq:att}\\
    \TF(\bH) &= \att(\bH) + \bW_2\relu(\bW_1\att(\bH)).\label{eq:relulayer}
\end{align}
\end{subequations}
Here $\sigma(\cdot)$ denotes the ReLU activation.
Consistent with previous literature, the first equation \eqref{eq:att} represents the attention layer, combining which with the feed-forward layer in \eqref{eq:relulayer} yields a single Transformer layer. 
We note here that the only difference with standard Transformer architecture is the elimination of the softmax operation and attention mask in the attention layer. 

A Transformer model can contain multiple Transformer layers defined as above, and the output of the whole model would be the composition of multiple layers.
For ease of presentation, we omit the details here.
From now on, we refer to Transformers with linear attention layers as \emph{linear Transformers}.

\subsection{In-context learning using Transformers}
In this work we consider two different settings for in-context learning. 
The first one is that of linear regression, which is commonly studied in the literature \citep{akyurek2022learning,bai2023transformers}, while we also go one step further and investigate the more difficult task of logistic regression. 
Our target would be to use the Transformer architecture to emulate in-context known algorithms for solving these tasks.

\subsubsection{Linear Regression}\label{ss:linear}

For the task of linear regression, 
let the  pairs $\{(\ba_i, y_i)\}_{i=1}^n$ be given as input to the Transformer, where $\ba_i\in\R^d$ and $y_i\in\R$ for all $i=1,\ldots,n$. 
We assume that for each such sequence of pairs, there is a weight vector $\bw_*$, such that $y_i = \bw_*^\top\ba_i+\epsilon_i$ for all $i=1,\hdots,n$, where $\epsilon_i$ is some noise. 
Given these samples, we want the Transformer to approximate the weight vector $\bw_*$ or make a new prediction on a test point $\ba_{test}$. 

Define $\by = (y_1, \ldots, y_n)^\top\in\RR^n$ and $\bA = [\ba_1, \ldots, \ba_n]^\top \in \RR^{n\times d}$.
The standard least-square solution is given by
\begin{equation}\label{eq:sol-linear}
    \hat\bw = (\bA^\top\bA)^{-1}\bA^\top\by.
\end{equation}
As a minimizer of the square loss $\ell(\bw) = \sum_{i=1}^n (y_i-\bw^\top\ba_i)^2$, $\hat\bw$ can also be obtained by minimizing $\ell(\bw)$ using, e.g., gradient descent.
Indeed, it has been shown that Transformers can perform gradient descent to solve linear regression in \citet{akyurek2022learning} and later in \citet{bai2023transformers} with explicit bounds on width and depth requirements. 
Specifically, in existing works, the number of steps (or equivalently, depth of the Transformer) required for convergence up to $\epsilon$ error is of order $\mathcal{O}(\kappa(\bA^\top\bA) \log(1/\epsilon))$. 
This suggests that the required number of layers increases linearly with the condition number.

Another approach is to compute directly the closed form solution \eqref{eq:sol-linear}, which involves the matrix inversion. 
One choice is Newton's method, an iterative method that approximates the inverse of a matrix with logarithmic dependence on the condition number and quadratic convergence to the accuracy improving upon gradient descent.

\paragraph{Newton's iteration for matrix inversion.}
The iteration can be described as follows:
Suppose we want to invert a matrix $\bA$, then with initialization $\bX_0 = \alpha\bA^\top$, we compute
\begin{align}\label{eq:inverse-newton}
    \bX_{t+1} &= \bX_t(2\id - \bA \bX_t)
\end{align}
For $\alpha \in (0,\frac{2}{\lambda_{\max}(\bA^\top\bA)})$, it can be shown that the estimate is $\epsilon$-accurate after $O(\log_2\kappa(\bA) + \log_2\log_2(1/\epsilon))$ steps \citep{ogden1969iterative, pan1991improved}.

One interesting generalization of the well-known Newton's iteration for matrix inversion is the following family of algorithms \citep{LI20103433}:
Initialized at $\bX_0 = \alpha \bA^\top$,
\begin{align}\label{eq:high_order_newton}
    \bX_{t+1}&=\bX_t \sum_{m=0}^{n-1}(-1)^m\binom{n}{m+1}(\bA\bX_t)^m.
\end{align}
We can see that for $n=2$ we get the standard Newton's iteration described in \cref{eq:inverse-newton}.
For any fixed $n\geq 2$, the corresponding algorithm has an $n$-th order convergence to the inverse matrix, i.e., $(\id - \bX_{k+1}\bA) = (\id -\bX_k\bA)^n$, suggesting that the error decays exponentially fast in an order of $n$.
This results in the improvement of the convergence rate by changing the logarithm basis from $\log_2$ to $\log_n$.
More importantly, the initial overhead of constant steps is also reduced, which is the main overhead of Newton's iteration.
This would become clear in the \cref{sec:experiments}, where we compare these methods against the Transformer architecture.

\subsubsection{Logistic Regression}\label{ss:logistic}
For in-context learning of logistic regression, we consider pairs of examples $\{(\ba_i,y_i)\}_{i=1}^n$ where each $\ba_i\in\R^d$ is the covariate vector and $y_i\in \braces{-1,1}$ is the corresponding label.  We assume that $y_i = \sign(\ba_i^\top\bw_*)$ for some vector $\bw_*\in\R^d$.
Our target is to find a vector $\hat\bw = \argmin_{\bw\in\RR^d} f(\bw)$ where $f:\RR^d\to\RR$ is the regularized logistic loss defined as
\begin{align}\label{eq:logistic_loss}
    f(\bw) := \dfrac{1}{n}\sum_{i=1}^n\log(1+\exp(-y_i\bw^\top\ba_i)) + \dfrac{\mu}{2}\norm{\bw}_2^2.
\end{align}
As in the setting of linear regression, we can use the vector $\hat\bw$ to make a new prediction on some point $\ba_{test}$ by calculating $1/(1+ \exp(-\hat{\bw}^\top\ba_{test}))$.

The $L_2$ penalty term is needed to ensure that the loss function is self-concordant in the following sense.
\begin{restatable}{definition}{defselfconcordant}\label{def:self-concordant}[Self-concordant function;  Definition 5.1.1, \citealt{nesterov2018lectures}]
Let $f:\RR^d\to \R$ be a closed convex function that is $3$ times continuously differentiable on its domain $\mathrm{dom}(f) := \{\bx\in\RR^d \mid f(\bx)<\infty\}$. 
For any fixed $\bx,\bu\in\RR^d$ and $t\in\R$, define $\phi(t;\bx,\bu) := f(\bx+t\bu)$ as a function of $t$.
Then we say $f$ is \emph{self-concordant} if there exists a constant $M_f$ such that, for all $\bx\in\mathrm{dom}(f)$ and $\bu\in\RR^d$ with $\bx+t\bu\in\mathrm{dom}(f)$ for all sufficiently small $t$, $$\abs{\phi'''(0;\bx,\bu)}\leq 2 M_f (\bu^\top\nabla^2 f(\bx)\bu)^{3/2}$$
We say $f$ is \emph{standard self-concordant} when $M_f=1$.
\end{restatable}

In particular, the regularized logistic loss is a self-concordant function under a mild assumption on the data.

\begin{assumption}\label{asm:bounded}
For the data $\{(\ba_i,y_i)\}_{i=1}^n$ in \eqref{eq:logistic_loss}, it holds that 
$\norm{\ba_i}_2\leq 1$ for all $i=1,\hdots,n$.
\end{assumption}

\begin{proposition}[Lemma 2, \citealt{zhang2015communicationefficient}]\label{prop:self-conc-constant}
For $\mu>0$, the regularized logistic loss $f(\cdot)$ defined in \eqref{eq:logistic_loss} is self-concordant with $M_f =1/\sqrt{\mu}$ under \cref{asm:bounded}.
Furthermore, the function $f/4\mu$ is standard self-concordant.
\end{proposition}

Self-concordance ensures rapid convergence of second-order optimization algorithms such as Newton's method. 
As in the case of matrix inversion, the rate of convergence is quadratic after a constant number of steps that depend on how close the algorithm was initialized to the minimum.

\paragraph{Newton's method.} 
Given the initialization $\bx_0\in\R^d$, the Newton's method updates as follows:
\begin{equation}\label{eq:Newton}
    \bx_{t+1} = \bx_t - \eta(\bx_t)\bracks{\nabla^2 f(\bx_t)}^{-1} \nabla f(\bx_t).
\end{equation}
Different choices of the step-size $\eta(\bx_t)$ lead to different variants of the algorithm: For $\eta = 1$ we have the ``classic'' Newton's method; for $\eta(\bx_t) = \frac{1}{1+\lambda(\bx_t)}$ with $\lambda(\bx_t) = \sqrt{\nabla f(\bx_t)^\top \bracks{\nabla^2 f(\bx_t)}^{-1} \nabla f(\bx_t)}$, we have the so-called damped Newton's method (see, e.g., Section 4 in \citet{nesterov2018lectures}).

For self-concordant functions, the damped Newton's method has guarantees for global convergence, which contains two phases:
In the initial phase, the quantity $\lambda(\bx)$ is decreased until it drops below the threshold of $1/4$. 
While $\lambda(\bx) \geq 1/4$, there is a constant decrease  per step of the self-concordant function $g$ of at least $0.02$. 
Then, the second phase begins once $\lambda(\bx)$ drops below the $1/4$ threshold, %
afterwards it will decay with a quadratic rate, implying a quadratic convergence of the objective value.
All together, this implies that  $c + \log\log(1/\epsilon)$ steps are required to reach $\epsilon$ accuracy.  For the analysis of the exact Newton's method for self-concordant functions one may refer to \cite{nesterov2018lectures}.

\section{Main Results for linear regression}\label{sec:results-linear}
We now present our main results on in-context learning for linear regression using Transformers.
The corresponding proofs of results in this section can be found in \Cref{app:linear-reg-con}.

Notice that in order to obtain the least-square solution in \eqref{eq:sol-linear}, it suffices to perform the operations of matrix inversion, matrix multiplications and matrix transposition. 
The linear Transformer architecture can trivially perform the last two operations, while our first result in \Cref{lem:inverse} below shows that it can also efficiently approximate matrix inversion via Newton's iteration.

\begin{restatable}{lemma}{leminverse}\label{lem:inverse}
For any dimension $d$, there exists a linear Transformer consisting of 2 linear attention layers, each of which has 2 attention heads and width $4d$, such that it can perform one step of Newton's iteration for any target matrix $\bA\in\RR^{d\times d}$.
Specifically, the Transformer realizes the following mapping from input to output for any $\bX_0\in\RR^{d\times d}$:
\begin{equation*}
    \begin{pmatrix}
        \bX_0\\
        \bA^\top\\
        \zero_d\\
        \id_d
    \end{pmatrix} \mapsto \begin{pmatrix}
        \bX_1\\
        \bA^\top\\
        \zero_d\\
        \id_d
    \end{pmatrix} 
\end{equation*}
where $\bX_1 = \bX_0(2\id_d - \bA\bX_0)$, corresponding to one step of Newton's iteration in \cref{eq:inverse-newton}.
Furthermore, if restricted to only symmetric $\bA$, then $1$ layer suffices.
\end{restatable} 

Built upon the construction from the above lemma, one can implement multiple steps of Newton's iteration by repeatedly applying such constructed layers.
This gives rise to the following main result on how linear Transformer can solve linear regression in-context.

\begin{restatable}[Linear regression]{theorem}{linreg}\label{thm:main_linear}
For any dimension $d, n$ and index $T>0$, there exists a linear Transformer consisting of $3+T$ layers, where each layer has 2 attention heads and width equal to $4d+3$, such that it realizes the following mapping from input to output:
\begin{align*}
    \begin{pmatrix}
    [\id_{d} \; \zero] \\
         [\id_{d} \; \zero] \\
          [\id_{d} \; \zero] \\
         \bA^\top\\
         \ba^\top_\mathrm{test},\zero \\
         \by^\top\\
         0,0,\hdots,0
    \end{pmatrix} \mapsto \begin{pmatrix}
    [\id_{d} \; \zero] \\
         [\id_{d} \; \zero] \\
          [\id_{d} \; \zero] \\
         \bA^\top\\
         \ba^\top_\mathrm{test},\zero \\
         \by^\top\\
         \hat{y},0,\ldots,0
    \end{pmatrix} 
\end{align*}
where $\bA = (\ba_1,\ldots,\ba_n)^\top$, $\by=(y_1,\ldots,y_n)^\top$, and $\hat y = \ba_\mathrm{test}^\top \bX_T \bA^\top \by$ is the prediction with $\bX_T$ being the output of $T$ steps of Newton's iteration for inversion on the matrix $\bA^\top\bA$, where the initialization is $\bX_0 = \epsilon \bA^\top\bA$ for some $\epsilon\in(0,\frac{2}{\lambda^2_{\max}(\bA^\top\bA)})$.
\end{restatable}

\begin{remark}
In \Cref{thm:main_linear},
the three extra layers are used to create the matrix $\bA^\top \bA$, perform the multiplication $ \bA \bX_T \ba_{test}$, and execute the final multiplication with $\by^\top$. 
\end{remark}

\begin{remark}
We note here that our construction corresponds to the solution of ridgeless regression, and it can be easily extended to the case of ridge regression, which amounts to inverting $\bA^\top\bA + \mu\id_d$ for some regularization parameter $\mu$.
\end{remark}

\begin{remark}
Considering previous results of implementing gradient descent for this task (e.g., \citealt{akyurek2022learning, bai2023transformers}), one may observe the memory trade-off between the two methods. 
In the case of gradient descent, there is no need of any extra memory, and each step can be described as $\bw_{t+1}^\top =\bw_t^\top -\eta( \by\bA - \bw_t^\top\bA^\top\bA)$ for some step-size $\eta$. 
Thus we can update the vector $\bw_t$ without storing the intermediate result $\bA^\top\bA$.
While applying Newton's iteration requires to store the intermediate results, as well as the initialization which results to the need of width equal to $4d$.
Nevertheless, the latter achieves better dependence on the condition number (logarithmic instead of linear) and quadratic decay of the error instead of linear.
\end{remark}

The recent  work by \citet{fu2023transformers} studies also Newton's iteration as a baseline to compare with Transformers in the setting of linear regression. They showed that Transformers can implement Newton's iteration, the stated result needs $\mathcal{O}(T)$ layers and $\mathcal{O}(d)$ width to perform $T$ steps of Newton's iteration on matrices of size $d$. %

Below in \cref{sec:experiments}, we will compare the predictions made by different orders of Newton's iteration \cref{eq:high_order_newton} with the Transformer architecture, as well as the acquired loss.

\section{Main Results for logistic regression}\label{sec:results-logistic}
In this section we will present constructive arguments showing that Transformers can approximately implement Newton's method for logistic regression.
To the best of our knowledge, the only result prior to our work for logistic regression is that of \citet{bai2023transformers} in which they implement gradient descent on the logistic loss using transformers. 
Newton's method can achieve a quadratic rate of convergence instead of linear with respect to the achieved accuracy.

As presented in \cref{ss:logistic}, we seek to minimize the regularized logistic loss defined in \eqref{eq:logistic_loss}, which is self-concordant by \Cref{prop:self-conc-constant}.
In a nutshell, our main result in this case shows that linear Transformer can efficiently emulate Newton's method to minimize the loss.
This is summarized in the following theorem.
We remind that $f$ is the regularized logistic loss of \cref{eq:logistic_loss}.

\begin{restatable}{theorem}{maintheorem}\label{thm:main_logistic}
For any dimension $d$, consider the regularized logistic loss defined in \eqref{eq:logistic_loss} with regularization parameter $\mu>0$, and define $\kappa_f = \max_{\bx\in\RR^d} \kappa(\nabla^2 f(\bx))$.
Then for any $T>0$ and $\epsilon>0$, there exists a linear Transformer that can approximate $T$ iterations of Newton's method on the regularized logistic loss up to error $\epsilon$ per iteration.  
In particular, the width of such a linear Transformer can be bounded by $O(d(1+\mu)^6/\epsilon^4\mu^8)$, and its depth can be bounded by $T(11+2k)$, where $k\leq 2\log\kappa_f + \log\log\frac{(1+\mu)^3}{\epsilon^2\mu^2}$.
Furthermore, there is a constant $c>0$ depending on $\mu$ such that if $\epsilon<c$, then
the output of the Transformer provides a $\tilde\bw$ satisfying that $\|\tilde\bw - \hat{\bw}\|_2 \leq O(\sqrt{\epsilon(1+\mu)/(4\mu)})$, where $\hat{\bw}$ is the global minimizer of the loss.%
\end{restatable}

The proof of \Cref{thm:main_logistic} contains two main components: The approximate implementation of Newton's method by Transformer and the convergence analysis of the resulting inexact Newton's method.
Below we will address these two parts separately in \Cref{sec:approximate_newton_logistic} and \Cref{sec:convergence_logistic}.

\subsection{Transformer can implement Newton's method for logisitic regression}\label{sec:approximate_newton_logistic}
To analyze the convergence properties of Newton's method on $f$, we actually implement the updates on $g = f/4\mu$, which based on \cref{prop:self-conc-constant} is standard self-concordant. 
Specifically, the targeted update is 
\begin{align*}
    \bx_{t+1} &= \bx_t - \dfrac{1}{1+\lambda_g(\bx_t)}(\nabla^2 g(\bx_t))^{-1} \nabla g(\bx_t)\\
    &= \bx_t - \dfrac{2\sqrt{\mu}}{2\sqrt{\mu} + \lambda_f(\bx_t)} (\nabla^2 f(\bx_t))^{-1}\nabla f(\bx_t)
\end{align*}
where the second equality follows directly from the definition $g = f / 4\mu$.
Our first result is that a linear Transformer can approximately implement the update above.
\begin{restatable}{theorem}{constrlogistic}
\label{thm:logistic}
Under the setting of \Cref{thm:main_logistic},
there exists a Transformer consisting of linear attention with ReLU layers that can approximately perform damped Newton's method on the regularized logistic loss as follows
\begin{align*}
    \bx_{t+1} = \bx_t -\dfrac{2\sqrt{\mu}}{2\sqrt{\mu}+\lambda(\bx_t)}(\nabla^2 f(\bx_t) )^{-1}\nabla f(\bx_t) + \bvarepsilon
\end{align*}
where $\bvarepsilon$ is an error term.
For any $\epsilon>0$, to achieve that $\|\bvarepsilon\|_2\leq \epsilon$, the width of such a Transformer can be bounded by $O(d\tfrac{(1+\mu)^8}{\epsilon^4\mu^{10}})$, and its depth can be bounded by $11+ 2k$ where $k\leq 2\log\kappa_f + \log\log\tfrac{(1+\mu)^3}{\epsilon^2\mu^2}$. 
\end{restatable}

\begin{remark} 
Here we omit the details of input and output format for ease of presentation.
See \cref{app:logistic} for details.
Roughly speaking, the input to the Transformer is of dimension $(6d+4) \times n$
and contains the matrix $\bA$, the labels $\by$, and the initialization $\bx_0$.

\end{remark}
Below we provide the main idea of the proof by describing main steps of our 
construction.

\begin{proof}[Proof sketch of \Cref{thm:logistic}]
To construct a Transformer that implements Newton's method to optimize the
regularized logistic loss, we implement first each one of the required 
components, including gradient $\nabla f(\bx)$, Hessian $\nabla^2 f(\bx)$, and 
the damping parameter $\lambda(\bx)$. 
The gradient and Hessian of $f$ are
\begin{align*}
    \nabla f(\bx) &= -\dfrac{1}{n}\sum_{i=1}^n y_ip_i\ba_i + \mu \bx,\\ 
    \nabla^2 f(\bx) &= \dfrac{1}{n}\bA^\top \bD\bA  + \mu\id_d\succeq \mu \id_d,
\end{align*}
where each $p_i := \exp(-y_i\bx^\top\ba_i)/(1+\exp(-y_i\bx^\top\ba_i))$ and 
$\bD := \diag(p_1(1-p_1), \ldots, p_n(1-p_n))$. 

\paragraph{Step 1: Approximate gradient and Hessian.} 
We use the linear attention layer to perform matrix multiplications, and we use
the ReLU network to approximate the following quantities: 
\begin{itemize}[left=0pt,itemsep = -2pt,topsep=-2pt]
    \item The values $p_i$.
    \item The diagonal elements of the matrix $\bD$.
    \item The ``hidden'' dot product of the diagonal elements of $\bD$ and the matrix $\bA$, since the $i$-th row of $\bD\bA$ is $d_i\ba_i^\top$.
\end{itemize}
These approximations give rise to the approximated gradient and Hessian of $f$, 
which we denote by $\hat{\nabla}f(\bx), \hat{\nabla}^2f(\bx)$.
\paragraph{Step 2: Invert the Hessian.}
For the next step, we use the construction presented in \cref{lem:inverse} to 
invert the matrix $\hat{\nabla}^2f(\bx)$. 
We leverage classical results in matrix perturbation theory to show that 
$(\hat{\nabla}^2f(\bx) )^{-1}= (\nabla^2 f(\bx))^{-1} + \bE_{1,t}$, 
where $\bE_{1,t}$ is an error matrix whose norm can be bounded
based on the number of iterations performed for the inversion.

\paragraph{Step 3: Approximate the step size.}
Next, we need to approximate the step size $\tfrac{2\sqrt{\mu}}{2\sqrt{\mu}+\lambda(\bx)}$.
This is done using the ReLU layers. 
Recall that $\lambda(\bx)^2 = \nabla f(\bx)^\top (\nabla^2f(\bx))^{-1}\nabla f(\bx)$, 
which can be approximated using $\hat\nabla f(\bx), \nabla^2 f(\bx)$ from the 
previous steps.
To get the step size, we need to further approximate the function $g(z) = 1/(1+\sqrt{z})$ and evaluate it at $\lambda(\bx)^2$. 
Thus, any error in the approximation of $\lambda(\bx)^2$ translates to a square 
root error in the calculation of $g(\lambda(\bx)^2)$. 
To see this, consider the derivative of the function $g$, $g'(z) = -2\sqrt{z}/(1+\sqrt{z})$, 
and observe that for $z$ of constant order, if $z$ changes by $\epsilon$, the corresponding change in $g(z)$ 
would be of order $\sqrt{\epsilon}$. 
This leads to the quadratic requirement of width with respect to the desired error threshold.

\paragraph{Step 4: Aggregate all approximations.}
Finally, we aggregate all the error terms induced by each approximation step and get the desired approximation to one step of Newton's method.
\end{proof}

\subsection{Convergence of inexact Newton's method}\label{sec:convergence_logistic}
\cref{thm:logistic} shows that linear Transformers can approximately implement damped Newton's method for logistic regression, while providing width requirements in order to control the approximation error. 
In complement to this, we further provide convergence analysis for the resulting
algorthm.

\begin{restatable}{theorem}{logisticonv}\label{thm:convergence-of-approximate-updates}
Under \Cref{asm:bounded}, for the regularized logistic loss $f$ defined in \eqref{eq:logistic_loss} with regularization parameter $\mu>0$,
consider a sequence of iterates $\{\bx_t\}_{t\geq 0}$ satisfying 
\begin{align*}
    \bx_{t+1} = \bx_t - \frac{2\sqrt{\mu}}{2\sqrt{\mu}+\lambda(\bx_t)} \nabla^2 f(\bx_t)^{-1} \nabla f(\bx_t) + \bvarepsilon_t
\end{align*}
where $\|\bvarepsilon_t\|_2\leq \epsilon$ for all 
$t\geq 0$.
Then there exists constants $c,C_1,C_2$ depending only on $\mu$ such that for any $\epsilon\leq c$, it holds that $g(\bx_t) - g(\bx^*) \leq \epsilon$ for all $t\geq C_1 + C_2 \log\log\frac{1}{\epsilon}$.
\end{restatable}

The proof of \cref{thm:convergence-of-approximate-updates} follows the proof presented in \citet{nesterov2018lectures} but accounts for the error term. 
Similar analysis has been performed in the past in \cite{sun2020composite}. %
Our proof consists of the following steps:
\begin{itemize}
\item \textbf{Feasibility}: We first show by induction that there exists a constant depending on $\mu$ such that $\|\bx_t\|_2\leq C$ for all $t\geq 0$. 
This is a consequence of the bounded norm assumption in \Cref{asm:bounded} and the strongly convex $L_2$ regularizer.
\item \textbf{Constant decrease}: We then show constant decreasing of the loss value per step when $\lambda(\bx_t) \geq 1/6$.
This is achieved by upper bounding the suboptimality of $g(\bx_t)$ a function of $\lambda(\bx_t)$, which is guaranteed by the self-concordance of $g$.
\item \textbf{Quadratic convergence}: In the last step, we show that when $\lambda(\bx_t) <1/6$, the inequality  $\lambda(\bx_{t+1}) \leq c\lambda(\bx_t)^2 + \epsilon'$ holds, which implies further quadratic convergence to achieving $O(\epsilon)$ error.
\end{itemize}

See \cref{app:logistic-Newtons} for the complete proof of \Cref{thm:convergence-of-approximate-updates}.
Finally, combining the construction in \Cref{thm:logistic} and the convergence 
analysis in \Cref{thm:convergence-of-approximate-updates} yields the performance 
guarantee of the constructed Transformer.

\section{Experiments}\label{sec:experiments}
In this section, we corroborate our theoretical findings with empirical evidence. 
Specifically, we aim to empirically demonstrate the effectiveness of the linear self-attention model when it is trained to solve linear regression
as well as logistic regression using encoder-based models. Our code is available here\footnote{\texttt{https://anonymous.4open.science/r/transformer\_higher\_order-B80B/}}.

\paragraph{Linear regression.}
For the task of linear regression, we train models consisting of linear self-attention (LSA) with and without LayerNorm (LN) to learn in-context. 
We use an embedding dimension of $64$ and 4 attention heads. The data dimension is $10$, and the input contains $50$ in-context samples. %
\begin{figure}
    \centering   
    \includegraphics[scale=0.5]{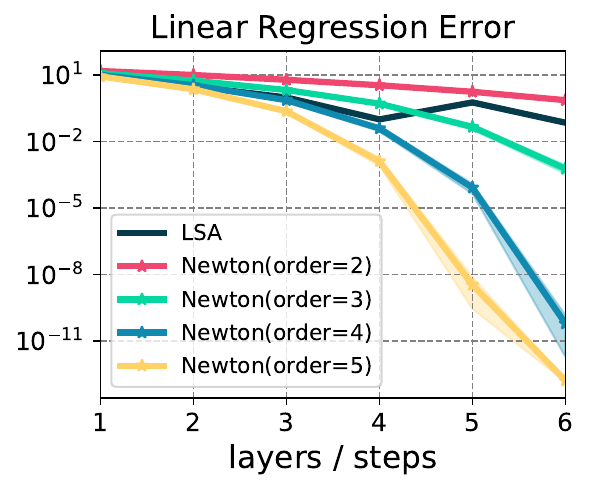}\hspace{3em}
    \includegraphics[scale = 0.5]{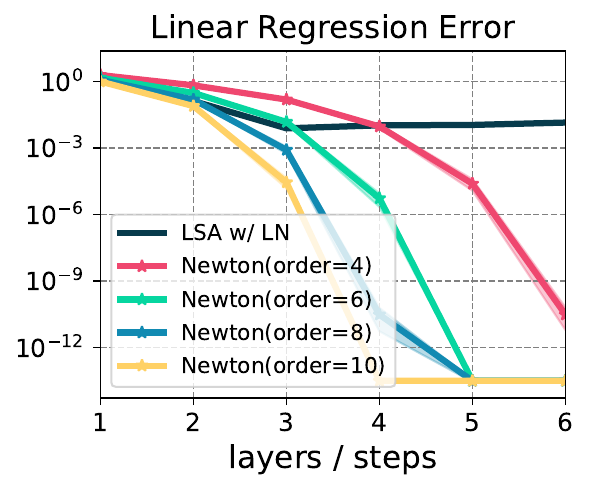}
    \caption{\small \emph{Loss of LSA, LSA with layernorm and different order Newton iteration for linear regression error.} }
    \label{fig:lsa-loss}
\end{figure}

We train models with LSA, having from $1$ to $6$ layers, employing the training technique of~\citet{vonoswald2023transformers} to stabilize the training process. Consistent with the findings reported in \citet{vonoswald2023transformers}, we observe that after four layers, there is no significant improvement in performance by adding more layers. 
We attribute this to limitations in the training process, and the optimal training method for linear Transformers beyond $5-6$ layers remains unidentified. 
We furthermore train LSA models with LN to be able to test more than $6$ layers, but to also test how LN affects the capabilities of these models. 

In \cref{fig:lsa-loss}, we plot the in-distribution loss (meaning we keep the same sampling method as in the training process) on new test samples. 
We observe that models having $4$ or more layers have almost the same performance in terms of the loss. 

To get a clearer picture of what the models are learning, we further provide plots illustrating the actual outputs of the trained Transformers, as shown in \cref{fig:lsa_func,fig:lsa_ln_func}.
Here we test the trained models in the following set-up: we first sample a batch of $5000$ examples, \ie $\{(\bx_i,y_i)\}_{i=1}^n$ data pairs that all share the same underlying weight vector $y_i = \bw_*^\top\bx_i$, for all $i=1,\hdots,n$ and all the batches. 
We then pick a specific $\bx_{test}$ sample, which we keep the same for all the batches. 
For this test sample, the value of its first coordinate varies across $[-a,a]$ for different values of $a$. 

We test two different cases: 1) the value $a$ has been encountered in the training set with probability at least $0.99$ and 2) the value is an outlier.  We calculate this value to be $[-14.9, 14.9]$,
which is the range of the `in-distribution' plots in \cref{fig:lsa_func,fig:lsa_ln_func} (range in between the dashed line), while in \Cref{fig:lsa_func_out,fig:lsa_ln_func_out} we have the `out-of-distribution' ones.

\begin{figure}[h]
\begin{subfigure}
    \centering    \includegraphics[width=0.95\textwidth]{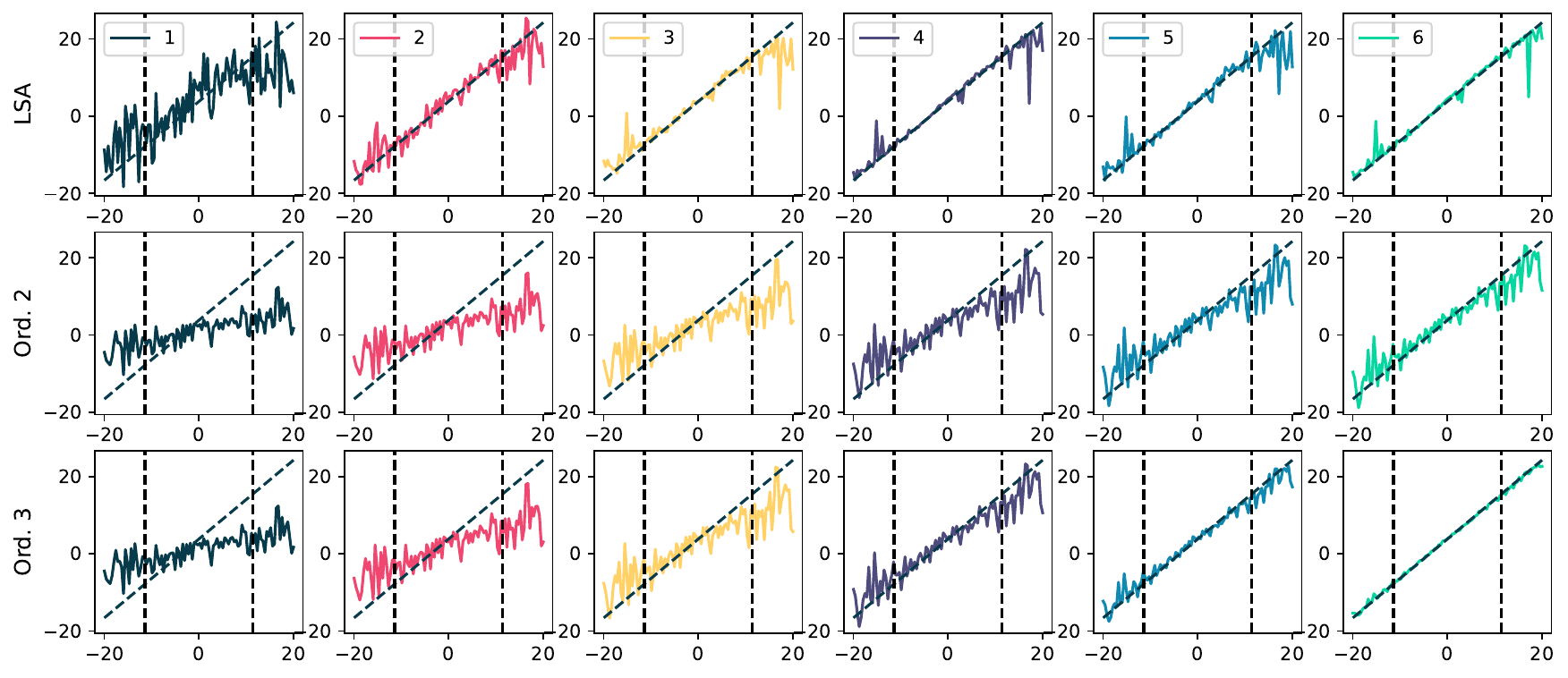}
    \vspace{-1em}
    \caption{\small \emph{Output of LSA without LayerNorm and Newton's iteration by keeping the test sample fixed and changing one of its coordinates, within in-distribution. We plot $1-6$ layers against $1-6$ steps of second and third order Newton's iteration ; we observe that the model lies between second and third order.}}
    \label{fig:lsa_func}
\end{subfigure}
\begin{subfigure}
    \centering
    \includegraphics[width=0.95\linewidth]{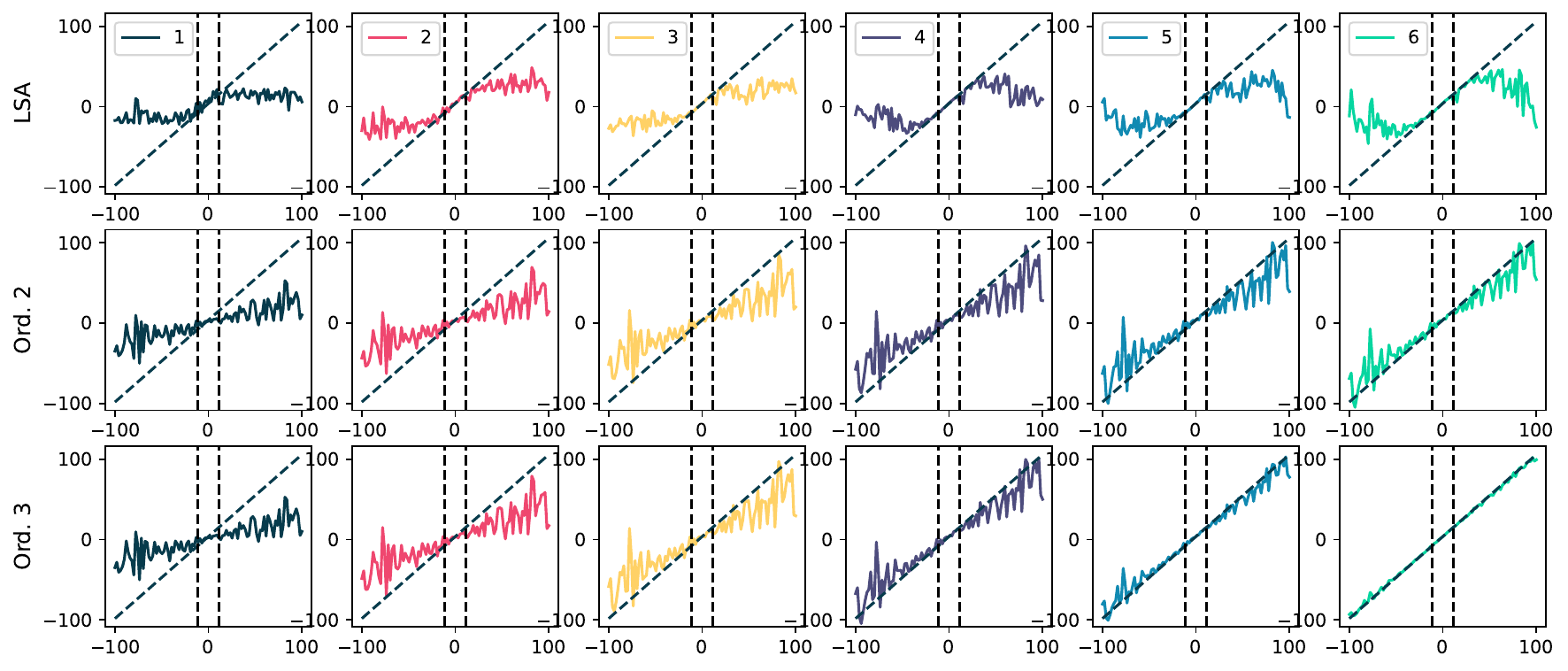}
    \vspace{-1em}
    \caption{\small \emph{Same setting as above, but the models/algorithms are tested with out-of-distribution values as well. We observe that the model does not actually learn the underlying linear function.}}
    \label{fig:lsa_func_out}
\end{subfigure}
\end{figure}

Looking at \cref{fig:lsa_func}, one may observe that the performance of the model lies between  the second- and third-order Newton's iteration.
Intuitively, Transformers do have slighter higher order capacity than  Newton's iteration. To see that, 
assume that the model is given in the input a symmetric matrix $\bA$, then in the first layer from \cref{eq:att} the model can create up to the third power of the matrix $\bA$, \ie all powers for $1-3$. %
Now in the second layer the higher-order term can be up to power of $9$; in contrast the second-order Newton's iteration can reach up to order $7$ (the first iteration $\bX_1 \sim \bA^3$, and in the second one $\bX_2 \sim \bX_1\bA\bX_1 \sim \bA^7$. 

The gap between the possible powers of the matrix, between second-order Newton's and Transformers will increase as the number of layers increases.  On the other hand, third order Newton's has up to power $\bA^5$ in the fist step $\bX_1$. While in second one the maximum power is $\bX_1(\bA\bX_1)^2 \sim \bA^{17}$. Thus, the LSA Transformer will not be able to outperform it.

\begin{figure}[h]
\begin{subfigure}
    \centering
    \includegraphics[width=0.95\linewidth]{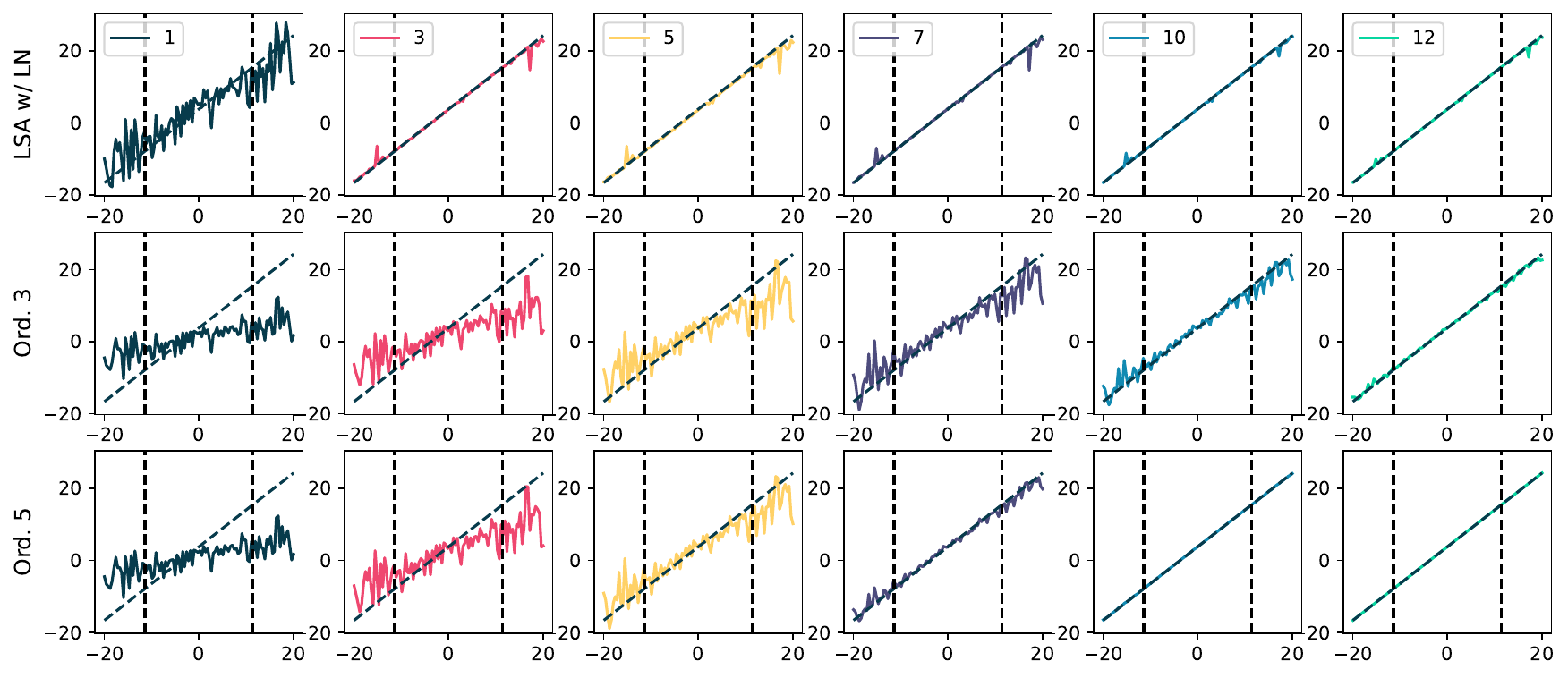}
    \vspace{-1em}
    \caption{\small \emph{Output of LSA with LayerNorm and Newton's iteration by keeping the test sample fixed and changing one of its coordinates, within in-distribution. We observe that LN improves the performance of the model. The spikes are noted in the out-of-distribution range. }} 
    \label{fig:lsa_ln_func}
\end{subfigure}
\begin{subfigure}
    \centering  \includegraphics[width=0.95\linewidth]{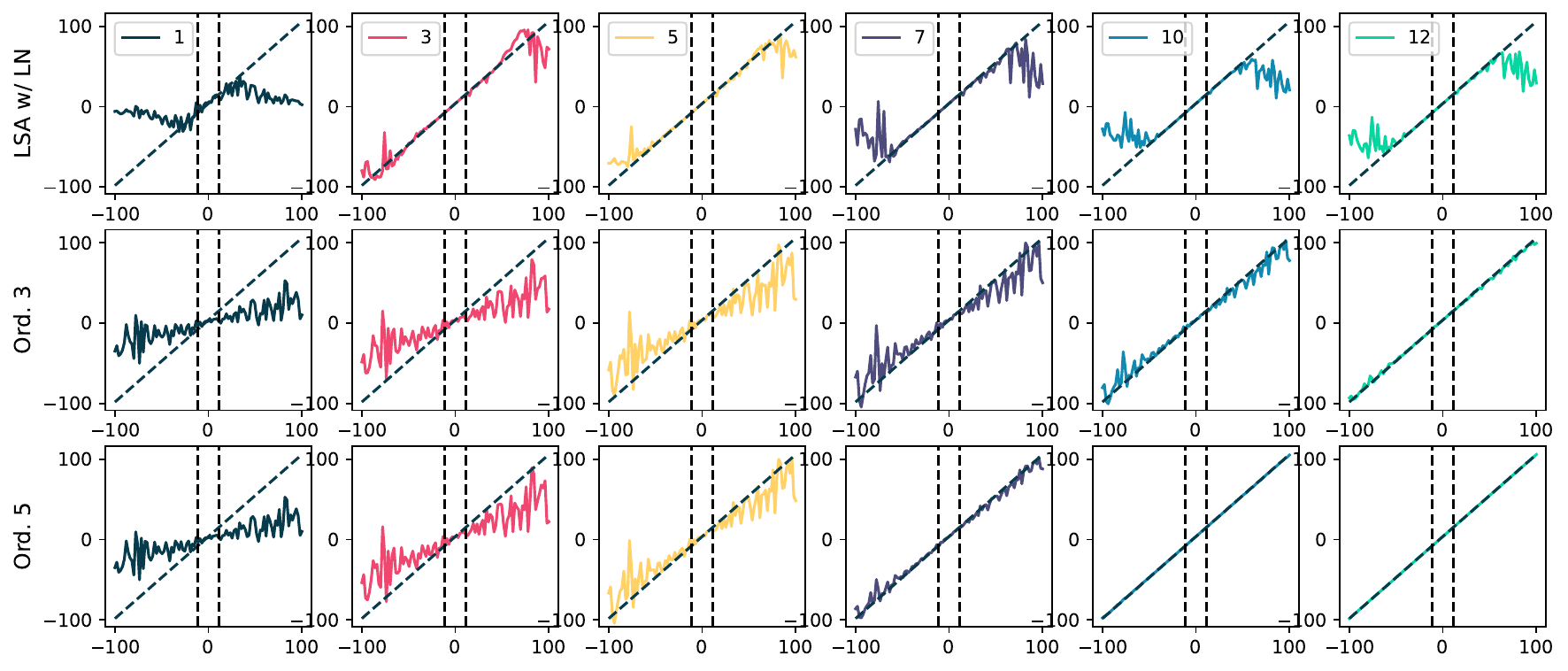}
    \vspace{-1em}
    \caption{\small \emph{Same setting as above, but the models/algorithms are tested with out-of-distribution values as well. $5$  layers seem to perform better for out-of-distribution than $7,10,12$, not though in-distribution.}}
    \label{fig:lsa_ln_func_out}
\end{subfigure}
\end{figure}

\paragraph{Logistic regression. }
We further analyze the Transformer's ability to solve logistic regression tasks. To simplify the setting, we focus on training the Transformer to predict the logistic regression parameter $w$. Since predicting the true weight directly is a hard task and is not necessarily the solution of the minimization problem described in \cref{eq:logistic_loss},
we instead opt to train the Transformer to output a solution comparable to that given by Newton's Method. 
\begin{figure}[H]
    \centering   \includegraphics[scale = 0.4]{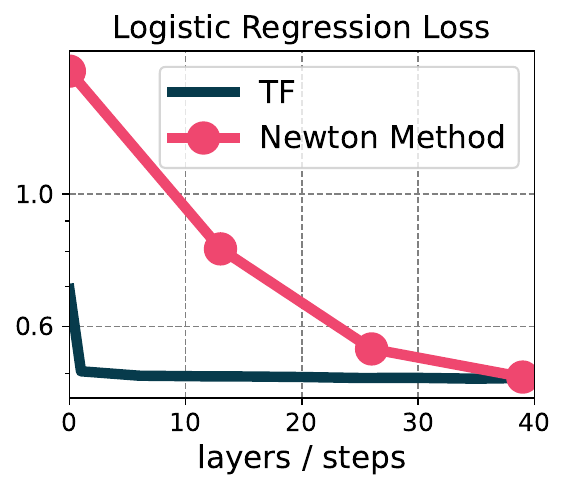}    
    \hspace{3em}
    \includegraphics[scale = 0.4] {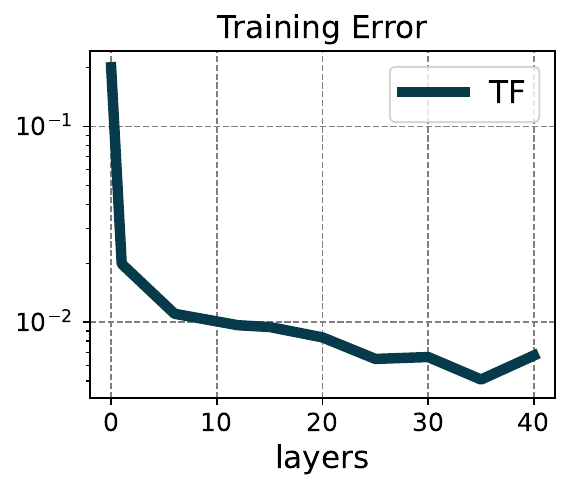}    
    \caption{\small \emph{Performance of Transformer on logistic regression tasks.} \emph{(Left)} The logistic regression loss for the Transformer (TF) and the Newton Method, with a regularization $\mu=0.1$. According to our theoretical construction, a single step of the Newton Method can be implemented by at least $11$ layers, therefore, we have scaled the Newton Method plot to $13$ layers per step. The Transformer is shown to approximate the method more effectively within a few layers. \emph{(Right)} The training error of the Transformer when it is trained to predict the solution derived from Newton's Method. 
    }   \label{fig:logistic_regression}
\end{figure}
We train on data dimension $5$, and the input contains $26$ in-context samples, and embedding dimension $32$, with $4$ heads. We set the regularization parameter $\mu = 0.1$ in the regularized logistic loss in \eqref{eq:logistic_loss}.
In \Cref{fig:logistic_regression}, we compare the loss value of the trained Transformer across different number of layers and the loss value of different iterates of Newton's method in the left plot; we further plot the corresponding training error in the right plot.
According to our construction in \Cref{thm:logistic}, a single step of the Newton Method is equivalent to 13 layers of a standard Transformer. 
However, when trained, the Transformer can approximate the Newton Method even more effectively as shown in \Cref{fig:logistic_regression}.

 \section{Discussion}\label{sec:discussion}
 In this paper we provide explicit construction of Transformers that can efficiently perform the matrix inversion operation.
Built upon this, we further show that Transformers can compute the least-square solution to solve the linear regression task in-context.
Moreover, for in-context learning of logistic regression, it also gives rise to construction of Transformers that can perform Newton's method to optimize the regularized logistic loss.
We provide concrete convergence analysis of the inexact Newton's method emulated by the constructed Transformer.
We further examine and compare the empirical performance of the trained Transformers with those of the higher-order methods.

%\bibliography{bibliography,refs}
\bibliographystyle{ims}

\newpage
\appendix

\section{Constructions of Transformers for linear regression}\label{app:linear-reg-con}

In this section we provide the exact constructions for implementing Newton's Iteration for matrix inversion. We then use this construction to approximate the closed form solution of linear regression in-context.

\subsection{Newton-Raphson Method for Matrix Inverse}
The Newton-Raphson method \citep{SchulzInverse33} for inverting a matrix $\bA\in\RR^{d\times d}$ is defined by the following update rule:
\begin{equation}\label{eq:app-newton-inverse}
    \bX_{t+1} = 2\bX_t - \bX_t\bA\bX_t.
\end{equation}
An initialization for this method that guarantees convergence is $\bX_0 = \epsilon \bA^\top$, where $\epsilon\in(0,\tfrac{2}{\lambda_1(\bA\bA^\top)})$.
We prove below that linear-attention transformers can emulate the above update.

\leminverse*

\begin{proof}
Assume that we are given the following input 
\begin{align*}
    \bH_0 = \begin{pmatrix}
        \bX_0\\
        \bA\\
        \zero\\
        \id_d
    \end{pmatrix}\in\R^{4d\times d}.
\end{align*}
For the first layer, we choose
\begin{align*}
    \bW_V = \begin{pmatrix}
        \zero_d & \zero_d & \zero_d & \zero_d\\
        \zero_d & \zero_d & \zero_d & \zero_d\\
        \zero_d & \id_d & \zero_d & \zero_d\\
        \zero_d & \zero_d & \zero_d & \zero_d
        \end{pmatrix},
    \bW_K = \begin{pmatrix}
        \zero_d & \zero_d & \zero_d & \id_d\\
        \zero_d & \zero_d & \zero_d & \zero_d\\
        \zero_d & \zero_d & \zero_d & \zero_d\\
        \zero_d & \zero_d & \zero_d & \zero_d
        \end{pmatrix},
    \bW_Q = \begin{pmatrix}
        \id_d & \zero_d & \zero_d & \zero_d\\
        \zero_d & \zero_d & \zero_d & \zero_d\\
        \zero_d & \zero_d & \zero_d & \zero_d\\
        \zero_d & \zero_d & \zero_d & \zero_d
        \end{pmatrix}
\end{align*}
so that 
\begin{align*}
    \bW_V\bH_0 = \begin{pmatrix}
        \zero_d\\
        \zero_d\\
        \bA\\
        \zero_d
    \end{pmatrix}, \quad
    \bW_K\bH_0 = \begin{pmatrix}
        \id_d\\
        \zero_d\\
        \zero_d\\
        \zero_d
    \end{pmatrix}, \quad
    \bW_Q\bH_0 = \begin{pmatrix}
        \bX_0\\
        \zero_d\\
        \zero_d\\
        \zero_d
    \end{pmatrix}.
\end{align*}
Then the output of the first layer is
\begin{align*}
    \bH_{1/2} = \bH_0 + \bW_V\bH_0 (\bW_K\bH_0)^\top \bW_Q\bH_0 
    = \begin{pmatrix}
        \bX_0\\
        \bA\\
        \bA\bX_0\\
        \id_d
        \end{pmatrix}.
\end{align*}

Next, we use two attention heads for the second layers.
For the first head, we choose 
\begin{align*}
    \bW_V^{(1)} = \begin{pmatrix}
        \id_d & \zero_d & \zero_d & \zero_d\\
        \zero_d & \zero_d & \zero_d & \zero_d\\
        \zero_d & \zero_d & \zero_d & \zero_d\\
        \zero_d & \zero_d & \zero_d & \zero_d
        \end{pmatrix},
    \bW_K^{(1)} = \begin{pmatrix}
        \zero_d & \zero_d & \zero_d & \id_d\\
        \zero_d & \zero_d & \zero_d & \zero_d\\
        \zero_d & \zero_d & \zero_d & \zero_d\\
        \zero_d & \zero_d & \zero_d & \zero_d
        \end{pmatrix},
    \bW_Q^{(1)} = \begin{pmatrix}
        \zero_d & \zero_d & -\id_d & \zero_d\\
        \zero_d & \zero_d & \zero_d & \zero_d\\
        \zero_d & \zero_d & \zero_d & \zero_d\\
        \zero_d & \zero_d & \zero_d & \zero_d
        \end{pmatrix}
\end{align*}
so that 
\begin{align*}
    \bW_V^{(1)}\bH_{1/2} = \begin{pmatrix}
        \bX_0\\
        \zero_d\\
        \zero_d\\
        \zero_d
    \end{pmatrix}, \quad
    \bW_K^{(1)}\bH_{1/2} = \begin{pmatrix}
        \id_d\\
        \zero_d\\
        \zero_d\\
        \zero_d
    \end{pmatrix}, \quad
    \bW_Q^{(1)}\bH_{1/2} = \begin{pmatrix}
        -\bA\bX_0\\
        \zero_d\\
        \zero_d\\
        \zero_d
    \end{pmatrix}.
\end{align*}
For the second head, we choose 
\begin{align*}
    \bW_V^{(2)} = \begin{pmatrix}
        \id_d & \zero_d & \zero_d & \zero_d\\
        \zero_d & \zero_d & \zero_d & \zero_d\\
        \zero_d & \zero_d & -\id_d & \zero_d\\
        \zero_d & \zero_d & \zero_d & \zero_d
        \end{pmatrix},
    \bW_K^{(2)} = \begin{pmatrix}
        \zero_d & \zero_d & \zero_d & \id_d\\
        \zero_d & \zero_d & \zero_d & \zero_d\\
        \zero_d & \zero_d & \zero_d & \zero_d\\
        \zero_d & \zero_d & \zero_d & \zero_d
        \end{pmatrix},
    \bW_Q^{(2)} = \begin{pmatrix}
        \zero_d & \zero_d & \zero_d & \id_d\\
        \zero_d & \zero_d & \zero_d & \zero_d\\
        \zero_d & \zero_d & \zero_d & \zero_d\\
        \zero_d & \zero_d & \zero_d & \zero_d
        \end{pmatrix}
\end{align*}
so that 
\begin{align*}
    \bW_V^{(2)}\bH_{1/2} = \begin{pmatrix}
        \bX_0\\
        \zero_d\\
        -\bA\bX_0\\
        \zero_d
    \end{pmatrix}, \quad
    \bW_K^{(2)}\bH_{1/2} = \begin{pmatrix}
        \id_d\\
        \zero_d\\
        \zero_d\\
        \zero_d
    \end{pmatrix}, \quad
    \bW_Q^{(2)}\bH_{1/2} = \begin{pmatrix}
        \id_d\\
        \zero_d\\
        \zero_d\\
        \zero_d
    \end{pmatrix}.
\end{align*}
Combining the outputs of the two heads, we get the output of the second layer as
\begin{align*}
    \bH_{1} &= \bH_{1/2} + \bW_V^{(1)}\bH_{1/2} (\bW_K^{(1)}\bH_{1/2})^\top \bW_Q^{(1)}\bH_{1/2} + \bW_V^{(2)}\bH_{1/2} (\bW_K^{(2)}\bH_{1/2})^\top \bW_Q^{(2)}\bH_{1/2}\\
    &= \bH_{1/2} + \begin{pmatrix}
        -\bX_0\bA\bX_0\\
        \zero\\
        \zero\\
        \zero
    \end{pmatrix} + \begin{pmatrix}
        \bX_0\\
        \zero\\
        -\bA\bX_0\\
        \zero
    \end{pmatrix}\\
    &= \begin{pmatrix}
        \bX_{1}\\
        \bA\\
        \zero\\
        \id_d
    \end{pmatrix}.
\end{align*}
This completes the proof. 
For the case where the matrix $\bA$ is symmetric, we give the construction as part of the proof for linear regression.
See the proof of \Cref{thm:main_linear} in \Cref{app:linear_proof}.
\end{proof}

\subsection{Linear regression through matrix inversion.}\label{app:linear_proof}
Given the least-square solution, the prediction for a fresh test point $\ba_\mathrm{test}$ is
\begin{align*}
    y_\mathrm{test} = \ba_\mathrm{test}^\top(\bA^\top\bA)^{-1}\bA^\top\by = \by^\top\bA(\bA^\top\bA)^{-1} \ba_\mathrm{test}.
\end{align*}
Below we give a construction of Transformer that can emulate the Newton-Raphson method to approximate the inverse of $\bA^\top\bA$ and then multiply the result with the necessary quantities to obtain an approximation of $y_\mathrm{test}$.
Notice that here the matrix we want to invert, $\bA^\top\bA$, is symmetric. 

\linreg*

\begin{proof}
We denote by $\bH_0$ the input to the Transformer.
To prove the desired result we present in steps the construction.

\paragraph{Step 1: Initialize (1 layer).}
We consider one Transformer layer with $2$ heads.
For the first head, we choose 
\begin{align*}
    \bW_V^{(1)} = \begin{pmatrix}
        \zero & \zero & \zero & \epsilon\id_d & \zero & \zero & \zero\\
        \zero & \zero & \zero & \id_d & \zero & \zero & \zero\\
        \zero & \zero & \zero & \zero & \zero & \zero & \zero\\
        \vdots & \vdots & \vdots & \vdots & \vdots & \vdots & \vdots\\
        \zero & \zero & \zero & \zero & \zero & \zero & \zero
    \end{pmatrix},
    \bW_K^{(1)} = \begin{pmatrix}
        \zero & \zero & \zero & \id_d & \zero & \zero & \zero\\
        \zero & \zero & \zero & \zero & \zero & \zero & \zero\\
        \zero & \zero & \zero & \zero & \zero & \zero & \zero\\
        \vdots & \vdots & \vdots & \vdots & \vdots & \vdots & \vdots\\
        \zero & \zero & \zero & \zero & \zero & \zero & \zero
    \end{pmatrix},
    \bW_Q^{(1)} = \begin{pmatrix}
        \id_d & \zero & \zero & \zero & \zero & \zero & \zero\\
        \zero & \zero & \zero & \zero & \zero & \zero & \zero\\
        \zero & \zero & \zero & \zero & \zero & \zero & \zero\\
        \vdots & \vdots & \vdots & \vdots & \vdots & \vdots & \vdots\\
        \zero & \zero & \zero & \zero & \zero & \zero & \zero
    \end{pmatrix}
\end{align*}
so that 
\begin{align*}
    \bW_V^{(1)}\bH_0 = \begin{pmatrix}
        \epsilon\bA^\top\\
        \bA^\top\\
        \zero\\
        \vdots\\
        \zero
    \end{pmatrix}, \quad
    \bW_K^{(1)}\bH_0 = \begin{pmatrix}
        \bA^\top\\
        \zero\\
        \zero\\
        \vdots\\
        \zero
    \end{pmatrix}, \quad
    \bW_Q^{(1)}\bH_0 = \begin{pmatrix}
        [\id_d\ \zero]\\
        \zero\\
        \zero\\
        \vdots\\
        \zero
    \end{pmatrix}
\end{align*}
For the second head, we choose
\begin{align*}
    \bW_V^{(2)} = \begin{pmatrix}
        -\id_d & \zero & \zero & \zero & \zero & \zero & \zero\\
        -\id_d & \zero & \zero & \zero & \zero & \zero & \zero\\
        \zero & \zero & \zero & \zero & \zero & \zero & \zero\\
        \vdots & \vdots & \vdots & \vdots & \vdots & \vdots & \vdots\\
        \zero & \zero & \zero & \zero & \zero & \zero & \zero
    \end{pmatrix},
    \bW_K^{(2)} = \begin{pmatrix}
        \id_d & \zero & \zero & \zero & \zero & \zero & \zero\\
        \zero & \zero & \zero & \zero & \zero & \zero & \zero\\
        \zero & \zero & \zero & \zero & \zero & \zero & \zero\\
        \vdots & \vdots & \vdots & \vdots & \vdots & \vdots & \vdots\\
        \zero & \zero & \zero & \zero & \zero & \zero & \zero
    \end{pmatrix},
    \bW_Q^{(2)} = \begin{pmatrix}
        \id_d & \zero & \zero & \zero & \zero & \zero & \zero\\
        \zero & \zero & \zero & \zero & \zero & \zero & \zero\\
        \zero & \zero & \zero & \zero & \zero & \zero & \zero\\
        \vdots & \vdots & \vdots & \vdots & \vdots & \vdots & \vdots\\
        \zero & \zero & \zero & \zero & \zero & \zero & \zero
    \end{pmatrix}
\end{align*}
so that
\begin{align*}
    \bW_V^{(2)}\bH_0 = \begin{pmatrix}
        [-\id_{d} \; \zero]\\
        [-\id_{d} \; \zero]\\
        \zero\\
        \vdots\\
        \zero
    \end{pmatrix}, \quad
    \bW_K^{(2)}\bH_0 = \begin{pmatrix}
        [\id_{d} \; \zero]\\
        \zero\\
        \zero\\
        \vdots\\
        \zero
    \end{pmatrix}, \quad
    \bW_Q^{(2)}\bH_0 = \begin{pmatrix}
        [\id_d\ \zero]\\
        \zero\\
        \zero\\
        \vdots\\
        \zero
    \end{pmatrix}.
\end{align*}
Then combining the above two heads, we have
\begin{align*}
    \bH_1 &= \bH_0 + \bW_V^{(1)}\bH_0 (\bW_K^{(1)}\bH_0)^\top \bW_Q^{(1)}\bH_0 + \bW_V^{(2)}\bH_0 (\bW_K^{(2)}\bH_0)^\top \bW_Q^{(2)}\bH_0\\
    &= \bH_0 - \begin{pmatrix}
         [-\id_{d} \; \zero] \\
         [-\id_{d} \; \zero] \\
          \zero\\
         \zero \\
         \zero\\
         \zero\\
         \zero
     \end{pmatrix} + \begin{pmatrix}
         [\epsilon\bA^\top\bA \; \zero] \\
         [\bA^\top\bA \; \zero] \\
          \zero\\
         \zero \\
         \zero\\
         \zero\\
         \zero
     \end{pmatrix}\\
     &= \begin{pmatrix}
         [\epsilon\bA^\top\bA \; \zero] \\
         [\bA^\top\bA \; \zero] \\
        [\id_{d} \; \zero] \\
         \bA^\top \\
        [\ba_\mathrm{test}^\top\;\zero]\\
         \by^\top\\
        0,0,\ldots,0
     \end{pmatrix}.
\end{align*}
\paragraph{Step 2: Implement $T$ steps of Newton-Raphson ($T$ layers).}
We now define $\bX_0 = \epsilon\bA^\top\bA$ and $\bR = \bA^\top\bA$. 
The input matrix to the next layer is in the following form (for $t=0$)
\begin{align*}
    \bH_t = \begin{pmatrix}
        [\bX_t \; \zero]\\
         [\bR \; \zero] \\
        [\id_{d} \; \zero] \\
         \bA^\top \\
         [\ba_\mathrm{test}^\top\;\zero]\\
         \by^\top\\
         0,\ldots,0
    \end{pmatrix}.
\end{align*}
We will show that one Transformer laye with two heads can yield the following output
\begin{align*}
    \bH_{t+1} = \begin{pmatrix}
         [\bX_{t+1} \; \zero]\\
         [\bR \; \zero] \\
        [\id_{d} \; \zero] \\
         \bA^\top \\
         [\ba_\mathrm{test}^\top\;\zero]\\
         \by^\top\\
         \zero
    \end{pmatrix}.
\end{align*}
Here we choose the weight matrices for the first head to be 
\begin{align*}
    \bW_V^{(1)} = \begin{pmatrix}
        -\id_d & \zero & \zero & \zero & \zero & \zero & \zero\\
        \zero & \zero & \zero & \zero & \zero & \zero & \zero\\
        \zero & \zero & \zero & \zero & \zero & \zero & \zero\\
        \vdots & \vdots & \vdots & \vdots & \vdots & \vdots & \vdots\\
        \zero & \zero & \zero & \zero & \zero & \zero & \zero
    \end{pmatrix},
    \bW_K^{(1)} = \begin{pmatrix}
        \id_d & \zero & \zero & \zero & \zero & \zero & \zero\\
        \zero & \zero & \zero & \zero & \zero & \zero & \zero\\
        \zero & \zero & \zero & \zero & \zero & \zero & \zero\\
        \vdots & \vdots & \vdots & \vdots & \vdots & \vdots & \vdots\\
        \zero & \zero & \zero & \zero & \zero & \zero & \zero
    \end{pmatrix},
    \bW_Q^{(1)} = \begin{pmatrix}
        \id_d & \zero & \zero & \zero & \zero & \zero & \zero\\
        \zero & \zero & \zero & \zero & \zero & \zero & \zero\\
        \zero & \zero & \zero & \zero & \zero & \zero & \zero\\
        \vdots & \vdots & \vdots & \vdots & \vdots & \vdots & \vdots\\
        \zero & \zero & \zero & \zero & \zero & \zero & \zero
    \end{pmatrix},
\end{align*}
so that 
\begin{align*}
    \bW_V^{(1)}\bH_t = \begin{pmatrix}
        [-\bX_{t} \; \zero]\\
        \zero\\
        \zero\\
        \vdots\\
        \zero
    \end{pmatrix}, \quad
    \bW_K^{(1)}\bH_t = \begin{pmatrix}
        [\bR \;\zero]\\
        \zero\\
        \zero\\
        \vdots\\
        \zero
    \end{pmatrix}, \quad
    \bW_Q^{(1)}\bH_t = \begin{pmatrix}
        [\bX_t\ \zero]\\
        \zero\\
        \zero\\
        \vdots\\
        \zero
    \end{pmatrix}.
\end{align*}
Similarly, for the second head, we choose
\begin{align*}
    \bW_V^{(2)} = \begin{pmatrix}
        \id_d & \zero & \zero & \zero & \zero & \zero & \zero\\
        \zero & \zero & \zero & \zero & \zero & \zero & \zero\\
        \zero & \zero & \zero & \zero & \zero & \zero & \zero\\
        \vdots & \vdots & \vdots & \vdots & \vdots & \vdots & \vdots\\
        \zero & \zero & \zero & \zero & \zero & \zero & \zero
    \end{pmatrix},
    \bW_K^{(2)} = \begin{pmatrix}
        \zero & \zero & \id_d & \zero & \zero & \zero & \zero\\
        \zero & \zero & \zero & \zero & \zero & \zero & \zero\\
        \zero & \zero & \zero & \zero & \zero & \zero & \zero\\
        \vdots & \vdots & \vdots & \vdots & \vdots & \vdots & \vdots\\
        \zero & \zero & \zero & \zero & \zero & \zero & \zero
    \end{pmatrix},
    \bW_Q^{(2)} = \begin{pmatrix}
        \zero & \zero & \id_d & \zero & \zero & \zero & \zero\\
        \zero & \zero & \zero & \zero & \zero & \zero & \zero\\
        \zero & \zero & \zero & \zero & \zero & \zero & \zero\\
        \vdots & \vdots & \vdots & \vdots & \vdots & \vdots & \vdots\\
        \zero & \zero & \zero & \zero & \zero & \zero & \zero
    \end{pmatrix}
\end{align*}
so that 
\begin{align*}
    \bW_V^{(2)}\bH_t = \begin{pmatrix}
        [\bX_{t} \; \zero]\\
        \zero\\
        \zero\\
        \vdots\\
        \zero
    \end{pmatrix}, \quad
    \bW_K^{(2)}\bH_t = \begin{pmatrix}
        [\id_d \;\zero]\\
        \zero\\
        \zero\\
        \vdots\\
        \zero
    \end{pmatrix}, \quad
    \bW_Q^{(2)}\bH_t = \begin{pmatrix}
        [\id_d\ \zero]\\
        \zero\\
        \zero\\
        \vdots\\
        \zero
    \end{pmatrix}.
\end{align*}
Combining these two heads, we obtain
\begin{align*}
    \bH_{t+1} &= \bH_t + \bW_V^{(1)}\bH_t (\bW_K^{(1)}\bH_t)^\top \bW_Q^{(1)}\bH_t + \bW_V^{(2)}\bH_t (\bW_K^{(2)}\bH_t)^\top \bW_Q^{(2)}\bH_t\\
    &= \bH_{t} + \begin{pmatrix}
        -[\bX_t\bR\bX_t\;\zero]\\
         \zero\\
        \vdots\\
         \zero
    \end{pmatrix} + \begin{pmatrix}
        [\bX_t\;\zero]\\
        \zero\\
        \vdots\\
        \zero
    \end{pmatrix}\\
    &=\begin{pmatrix}
          [2\bX_t -\bX_t\bR\bX_t \;\zero]\\
          [\bR \; \zero] \\
        [\id_{d} \; \zero] \\
         \bA^\top \\
         [\ba_\mathrm{test}^\top\;\zero]\\
         \by^\top\\
         0,\ldots,0
     \end{pmatrix}\\
     &= \begin{pmatrix}
          [\bX_{t+1} \;\zero]\\
          [\bR \; \zero] \\
        [\id_{d} \; \zero] \\
         \bA^\top \\
         [\ba_\mathrm{test}^\top\;\zero]\\
         \by^\top\\
         0,\ldots,0
     \end{pmatrix}.
\end{align*}
Repeating the above for $T$ many layers yields 
\begin{align*}
    \bH_{T+1} = \begin{pmatrix}
          [\bX_{T} \;\zero]\\
          [\bR \; \zero] \\
        [\id_{d} \; \zero] \\
         \bA^\top \\
         [\ba_\mathrm{test}^\top\;\zero]\\
         \by^\top\\
         0, \ldots, 0
     \end{pmatrix}.
\end{align*}

\paragraph{Step 3: Output (2 layers).}
We now create first the matrix $\by^\top\bA\bX_T$ with one Transformer layer and then the final prediction output with another layer. 
For the first layer, we choose 
\begin{align*}
    \bW_V = \begin{pmatrix}
        \zero & \zero & \zero & \zero & \zero & \zero & \zero\\
        \zero & \zero & \zero & \zero & \zero & \zero & \zero\\
        \vdots & \vdots & \vdots & \vdots & \vdots & \vdots & \vdots\\
        \zero & \zero & \zero & \zero & \zero & \zero & \zero\\
        \zero & \zero & \zero & \zero & \zero & \bm{1} & \zero
    \end{pmatrix},
    \bW_K = \begin{pmatrix}
        \zero & \zero & \id_d & \zero & \zero & \zero & \zero\\
        \zero & \zero & \zero & \zero & \zero & \zero & \zero\\
        \zero & \zero & \zero & \zero & \zero & \zero & \zero\\
        \vdots & \vdots & \vdots & \vdots & \vdots & \vdots & \vdots\\
        \zero & \zero & \zero & \zero & \zero & \zero & \zero
    \end{pmatrix},
    \bW_Q = \begin{pmatrix}
        \zero & \zero & \id_d & \zero & \zero & \zero & \zero\\
        \zero & \zero & \zero & \zero & \zero & \zero & \zero\\
        \zero & \zero & \zero & \zero & \zero & \zero & \zero\\
        \vdots & \vdots & \vdots & \vdots & \vdots & \vdots & \vdots\\
        \zero & \zero & \zero & \zero & \zero & \zero & \zero
    \end{pmatrix}
\end{align*}
so that 
\begin{align*}
    \bW_V\bH_{T+1} = \begin{pmatrix}
        \zero\\
        \zero\\
        \vdots\\
        \zero\\
        \by^\top
    \end{pmatrix}, \quad
    \bW_K\bH_{T+1} = \begin{pmatrix}
        \bA^\top\\
        \zero\\
        \zero\\
        \vdots\\
        \zero
    \end{pmatrix}, \quad
    \bW_Q\bH_{T+1} = \begin{pmatrix}
        [\bX_T\ \zero]\\
        \zero\\
        \zero\\
        \vdots\\
        \zero
    \end{pmatrix}.
\end{align*}
The output of this layer is
\begin{align*}
    \bH_{T+2} = \bH_{T+1} + \bW_V\bH_{T+1} (\bW_K\bH_{T+1})^\top \bW_Q\bH_{T+1}
    =\bH_{T+1} + \begin{pmatrix}
        \zero\\
        \zero\\
         \zero \\
         \zero\\
         \zero\\
         \zero\\
        [\by^\top\bA\bX_T\; \zero]
    \end{pmatrix} = \begin{pmatrix}
        [\bX_T\ \zero]\\
        [\bR\ \zero]\\
        [\id_d\;\zero] \\
        \bA^\top\\
        [\ba_\mathrm{test}^\top\;\zero]\\
        \by^\top\\
        [\by^\top\bA\bX_T\; \zero]
    \end{pmatrix}.
\end{align*}

Now for the final step, we construct a Transformer layer with two heads.
For the first head, we choose
\begin{align*}
    \bW_V^{(1)} = \begin{pmatrix}
        \zero & \zero & \zero & \zero & \zero & \zero & \zero\\
        \zero & \zero & \zero & \zero & \zero & \zero & \zero\\
        \vdots & \vdots & \vdots & \vdots & \vdots & \vdots & \vdots\\
        \zero & \zero & \zero & \zero & \zero & \zero & \zero\\
        \zero & \zero & \zero & \zero & \zero & \bm{1} & \zero
    \end{pmatrix},
    \bW_K^{(1)} = \begin{pmatrix}
        \zero & \zero & \zero & \zero & 1 & \zero & \zero\\
        \zero & \zero & \zero & \zero & \zero & \zero & \zero\\
        \zero & \zero & \zero & \zero & \zero & \zero & \zero\\
        \vdots & \vdots & \vdots & \vdots & \vdots & \vdots & \vdots\\
        \zero & \zero & \zero & \zero & \zero & \zero & \zero
    \end{pmatrix},
    \bW_Q^{(1)} = \begin{pmatrix}
        \zero & \zero & 1 & \zero & \zero & \zero & \zero\\
        \zero & \zero & \zero & \zero & \zero & \zero & \zero\\
        \zero & \zero & \zero & \zero & \zero & \zero & \zero\\
        \vdots & \vdots & \vdots & \vdots & \vdots & \vdots & \vdots\\
        \zero & \zero & \zero & \zero & \zero & \zero & \zero
    \end{pmatrix}
\end{align*}
so that 
\begin{align*}
    \bW_V^{(1)}\bH_{T+2} = \begin{pmatrix}
        \zero\\
        \zero\\
        \vdots\\
        \zero\\
        [\by^\top\bA\bX_T\;\zero]
    \end{pmatrix}, \quad
    \bW_K^{(1)}\bH_{T+2} = \begin{pmatrix}
        [\ba_\mathrm{test}^\top\;\zero]\\
        \zero\\
        \zero\\
        \vdots\\
        \zero
    \end{pmatrix}, \quad
    \bW_Q^{(1)}\bH_{T+2} = \begin{pmatrix}
        [1\ \zero]\\
        \zero\\
        \zero\\
        \vdots\\
        \zero
    \end{pmatrix}.
\end{align*}
For the second head, we choose
\begin{align*}
    \bW_V^{(2)} = \begin{pmatrix}
        \zero & \zero & \zero & \zero & \zero & \zero & \zero\\
        \zero & \zero & \zero & \zero & \zero & \zero & \zero\\
        \vdots & \vdots & \vdots & \vdots & \vdots & \vdots & \vdots\\
        \zero & \zero & \zero & \zero & \zero & \zero & \zero\\
        \zero & \zero & \zero & \zero & \zero & -\bm{1} & \zero
    \end{pmatrix},
    \bW_K^{(2)} = \begin{pmatrix}
        \zero & \zero & \id_d & \zero & \zero & \zero & \zero\\
        \zero & \zero & \zero & \zero & \zero & \zero & \zero\\
        \zero & \zero & \zero & \zero & \zero & \zero & \zero\\
        \vdots & \vdots & \vdots & \vdots & \vdots & \vdots & \vdots\\
        \zero & \zero & \zero & \zero & \zero & \zero & \zero
    \end{pmatrix},
    \bW_Q^{(2)} = \begin{pmatrix}
        \zero & \zero & \id_d & \zero & \zero & \zero & \zero\\
        \zero & \zero & \zero & \zero & \zero & \zero & \zero\\
        \zero & \zero & \zero & \zero & \zero & \zero & \zero\\
        \vdots & \vdots & \vdots & \vdots & \vdots & \vdots & \vdots\\
        \zero & \zero & \zero & \zero & \zero & \zero & \zero
    \end{pmatrix}
\end{align*}
so that 
\begin{align*}
    \bW_V^{(2)}\bH_{T+2} = \begin{pmatrix}
        \zero\\
        \zero\\
        \vdots\\
        \zero\\
        [\by^\top\bA\bX_T\;\zero]
    \end{pmatrix}, \quad
    \bW_K^{(2)}\bH_{T+2} = \begin{pmatrix}
        [\id_d\;\zero]\\
        \zero\\
        \zero\\
        \vdots\\
        \zero
    \end{pmatrix}, \quad
    \bW_Q^{(2)}\bH_{T+2} = \begin{pmatrix}
        [\id_d\;\zero]\\
        \zero\\
        \zero\\
        \vdots\\
        \zero
    \end{pmatrix}.
\end{align*}
Putting these together, we obtain
\begin{align*}
    \bH_{T+3} &= \bH_{T+2} + \begin{pmatrix}
        \zero\\
       \zero\\
        \zero \\
        \zero\\
        [\by^\top\bA\bX_T\ba_\mathrm{test}\ \zero]
   \end{pmatrix} + 
        \begin{pmatrix}
    \zero\\
        \zero\\
         \zero \\
         \zero\\
        -  [\by^\top\bA\bX_T\;\zero]
    \end{pmatrix}
    =\begin{pmatrix}
        [\bX_T\ \zero]\\
        [\bR\ \zero]\\
        [\id_d\ \zero]\\
        \bA^\top\\
        [\ba_\mathrm{test}^\top\;\zero]\\
         \by^\top\\
         [\by^\top\bA\bX_T\ba_{test}\ \zero]
    \end{pmatrix}
\end{align*}
Notice that the last element of the last row is the desired quantity $\hat{y} =  \by^\top\bA\bX_t\ba_\mathrm{test}$.
This completes the proof.
\end{proof}

\section{Convergence of inexact damped Newton's method for regularized logistic regression}
We first review some basics for self-concordant functions in \Cref{app:self-conc},
and then provide the proof of \Cref{thm:convergence-of-approximate-updates} in \Cref{app:logistic-Newtons}.

\subsection{Preliminaries on self-concordant functions}\label{app:self-conc}
We review some results on self-concordant functions that will be useful for the next sections. 
Most of these theorems can be found in \cite{boyd2004convex,nesterov2018lectures}. 

First recall the definition of self-concordant functions.
\defselfconcordant*

For the regularized logistic regression problem that we consider, the objective function is strongly convex, and the theorems provided below use the strong convexity. 
Nonetheless, these theorems have more general forms, as detailed in Chapter 5 of \cite{nesterov2018lectures}.

\newcommand{\dom}{\mathrm{dom}}
In the sequel, let $f$ be a self-concordant function.

\begin{definition}[Dikin ellipsoid]
For a function $f:\RR^d\to\RR$, consider the following sets for any $\bx\in\RR^d$ and $r>0$:
\begin{align*}
    \mathcal{W}(\bx;r) &= \braces{\by\in\R^d: \norm{\by-\bx}_{\nabla^2 f(\bx)}<r}%
\end{align*}
where $cl(\cdot)$ defines the closure of a set.
This set is called the \emph{Dikin ellipsoid} of the function $f$ at $\bx$.
\end{definition}

\begin{theorem}[Theorem 5.1.5 in \cite{nesterov2018lectures}]\label{thm:dikin}
Let $f:\RR^d\to\RR$ be a self-concordant function.
Then for any $\bx\in\dom(f)$, it holds that $\mathcal{W}(\bx;1/M_f) \subseteq \dom(f)$ .
\end{theorem}

\begin{theorem}[Theorem 5.1.8 \& 5.1.9 in \cite{nesterov2018lectures}]\label{thm:constant-decrement}
Let $f:\RR^d\to\RR$ be a self-concordant function.
Then for any $\bx, \by \in \dom(f)$, it holds that
\begin{align*}
    f(\bx) + \inner{\nabla f(\bx)}{\by-\bx} + \frac{\omega(M_f \norm{\by-\bx}_{\nabla^2 f(\bx)})}{M_f^2} \leq 
         f(\by) \leq f(\bx) + \inner{\nabla f(\bx)}{\by-\bx} + \frac{\omega_*(M_f\norm{\by-\bx}_{\nabla^2 f(\bx)})}{M_f^2}
\end{align*}
where $\omega(t) = t -\ln(1+t)$, $\omega_*(t) = -t-\ln(1-t)$. 
\end{theorem}

\begin{lemma}[Lemma 5.1.5 in \cite{nesterov2018lectures}]\label{lem:bound-on-omega}
For any $t\geq 0$, the functions $\omega(t),\omega_*(t)$ in \Cref{thm:constant-decrement} satisfy:
\begin{align*}
    \dfrac{t^2}{2(1+t)} \leq \dfrac{t^2}{2(1+2t/3)}\leq \omega(t) \leq \dfrac{t^2}{2+t}, \quad
    \dfrac{t^2}{2-t} \leq \omega_*(t) \leq \dfrac{t^2}{2(1-t)}.
\end{align*}
\end{lemma}

\begin{theorem}[Theorem 5.1.7, \citealt{nesterov2018lectures}]\label{thm:cound-hessian}
Let $f:\RR^d\to\RR$ be a self-concordant function.
Then for any $\bx\in\dom(f)$ and any $\by\in\mathcal{W}(\bx;1/M_f)$, it holds that
\begin{align*}
    (1-M_fr)^2 \cdot \nabla^2f(\bx) \preceq \nabla^2 f(\by)\preceq \dfrac{1}{(1-M_fr)^2} \cdot \nabla^2f(\bx).
\end{align*}
where $r=\|\by-\bx\|_{\nabla^2 f(\bx)}$.
\end{theorem}
A useful corollary of the above theorem is
\begin{corollary}[Corollary 5.1.5, \citealt{nesterov2018lectures}]\label{cor:bound-hessian}
Let $f:\RR^d\to\RR$ be a self-concordant function.
Then for any $\bx\in\dom(f)$ and any $\by$ such that $r = \|\by-\bx\|_{\nabla^2 f(\bx)} < 1/M_f$, it holds that
\begin{align*}
    \bigg(1-M_fr +\dfrac{1}{3}M_f^2r^2\bigg) \cdot \nabla^2 f(\bx) \preceq \int_{0}^1 \nabla^2f(\bx+\tau(\by-\bx))\mathrm{d}\tau \preceq \dfrac{1}{1-M_fr} \cdot \nabla^2 f(\bx).
\end{align*}
\end{corollary}

A key quantity in the analysis of Newton method for self-concordant functions is the so-called Newton decrement:
\begin{equation}\label{eq:Newton-decrement}
    \lambda_f(\bx) = (\nabla f(\bx)^\top \nabla^2 f(\bx)^{-1} \nabla f(\bx))^{1/2}
\end{equation}
The following theorem characterizes the sub-optimality gap in terms of the Newton decrement for a strongly convex self-concordant function.
\begin{theorem}[Theorem 5.1.13 in \cite{nesterov2018lectures}]\label{thm:convergence to optimum}
Let $f:\RR^d\to\RR$ be a strongly convex self-concordant function with $\bx^* = \argmin_{\bx}f(\bx)$.
Suppose $\lambda_f(\bx) <1/M_f$ for some $\bx\in\dom(f)$,
then the following holds:
\begin{align*}
    f(\bx)-f(\bx^*) \leq \dfrac{1}{M_f^2} \cdot \omega_*(M_f\lambda_f(\bx))
\end{align*}
where $\omega_*(\cdot)$ is the function defined in \Cref{thm:constant-decrement}.
\end{theorem}

\subsection{Convergence analysis of inexact damped Newton's method}\label{app:logistic-Newtons}
In this section we consider the inexact damped Newton's method for optimizing the regularized logistic loss:
\begin{align}\label{eq:app-inexact-updates}
    \hat{\bx}_{t+1} &= \hat{\bx}_t- \eta(\hat{\bx}_t)(\nabla^2f(\hat{\bx}_t))^{-1}\nabla f(\hat{\bx}_t) + \bvarepsilon_t
\end{align}
where $\bvarepsilon_t$ is an error term.
For simplicity, we also define 
\begin{align}\label{eq:Delta}
    \bDelta_t := - \eta(\hat{\bx}_t)(\nabla^2f(\hat{\bx}_t))^{-1}\nabla f(\hat{\bx}_t) + \bvarepsilon_t.
\end{align}
In other words, we have $\hat{\bx}_{t+1} = \hat{\bx}_t + \bDelta_t$.

Recall that the regularized logistic loss defined in \Cref{eq:logistic_loss} is self-concordant with $M_f = 1/\sqrt{\mu}$, as a consequence of \Cref{prop:self-conc-constant}.

\begin{lemma}\label{cor:gradient-bounds}
Under \Cref{asm:bounded}, let $f$ be the regularized logistic loss defined in \Cref{eq:logistic_loss} with regularization parameter $\mu>0$.
Then the following bounds hold 
\begin{align*}
    -1 + \mu\|\bx\|_2 \leq \norm{\nabla f(\bx)}_2 &\leq 1 + \mu\|\bx\|_2,\\
    \mu \leq \norm{\nabla^2 f(\bx)}_\op &\leq 1 + \mu.
\end{align*}
Consequently, for the Newton decrement $\lambda(\bx) = \sqrt{\nabla f(\bx)^\top (\nabla^2 f(\bx))^{-1} \nabla f(\bx)}$, it holds that
\begin{align*}
    \frac{-1+\mu\|\bx\|_2}{\sqrt{1+\mu}} \leq \lambda(\bx) \leq \dfrac{1+\mu\|\bx\|_2}{2\sqrt{\mu}}.
\end{align*}
\end{lemma}
\begin{proof}[Proof of \Cref{cor:gradient-bounds}]
For $|f(\bx)|$, note that for each $i\in[n]$, we have $-y_i\bx^\top \ba_i \leq \|\bx\|_2$ because $y_i\in\{-1,1\}$ and $\|\ba_i\|_2\leq 1$.
This implies the first bound on $|f(\bx)|$.

Next, recall the gradient and Hessian of $f$:
\begin{align*}
    \nabla f(\bx) = -\dfrac{1}{n}\sum_{i=1}^n y_ip_i\ba_i + \mu \bx,\quad 
    \nabla^2 f(\bx) = \dfrac{1}{n}\bA^\top \bD\bA  + \mu\id_d,
\end{align*}
where each $p_i := \exp(-y_i\bx^\top\ba_i)/(1+\exp(-y_i\bx^\top\ba_i))$ and 
$\bD := \diag(p_1(1-p_1), \ldots, p_n(1-p_n))$. 
By triangle inequality, we have
\begin{align*}
    \|\nabla f(\bx)\|_2 \leq \dfrac{1}{n}\sum_{i=1}^n |y_ip_i|\|\ba_i\|_2 + \mu\|\bx\|_2 \leq 1 + \mu\|\bx\|_2
\end{align*}
where the last inequality follows from \Cref{asm:bounded}.
Similarly, for the Hessian, we have
\begin{align*}
    \|\nabla^2 f(\bx)\|_\op &\leq \dfrac{1}{n}\|\bA^\top \bD\bA\|_\op + \mu \leq \dfrac{1}{n}\|\bA\|_\op^2 \|\bD\|_\op + \mu \leq 1 + \mu.
\end{align*}
Finally, the lower bound on $\|\nabla^2 f(\bx)\|_\op$ follows from the fact that $\frac{1}{n}\bA^\top \bD\bA$ is positive definite.
This completes the proof.
\end{proof}

\begin{lemma}\label{lem:norm_descent}
Under \Cref{asm:bounded}, let $f$ be the regularized logistic loss defined in \Cref{eq:logistic_loss} with regularization parameter $\mu>0$.
Then there exists some constant $C>0$ depending on $\mu$ such that for any $\bvarepsilon\in\RR^d$ with $\|\bvarepsilon\|_2\leq \mu^2$, if $\|\bx\|_2\geq C$, then 
\begin{align*}
    \bigg\|\bx - \frac{2\sqrt{\mu}}{2\sqrt{\mu}+\lambda(\bx)} \nabla^2 f(\bx)^{-1} \nabla f(\bx) + \bvarepsilon\bigg\|_2 \leq \|\bx\|_2.
\end{align*}
\end{lemma}
\begin{proof}[Proof of \Cref{lem:norm_descent}]
For simplicity, denote 
\begin{align*}
    \bx' = \bx - \frac{2\sqrt{\mu}}{2\sqrt{\mu}+\lambda(\bx)} \nabla^2 f(\bx)^{-1} \nabla f(\bx) + \bvarepsilon, \quad \eta(\bx) = \frac{2\sqrt{\mu}}{2\sqrt{\mu}+\lambda(\bx)}.
\end{align*}
It follows that 
\begin{align}
    \|\bx'\|_2^2 &= \|\bx\|_2^2 - 2\eta(\bx) \bx^\top \nabla^2 f(\bx)^{-1} \nabla f(\bx) + \eta(\bx)^2 \|\nabla^2 f(\bx)^{-1} \nabla f(\bx)\|_2^2\notag\\
    &\qquad +2\bvarepsilon^\top (\bx - \eta(\bx)\nabla^2 f(\bx)^{-1}\nabla f(\bx))+ \|\bvarepsilon\|_2^2\notag\\
    &\leq \|\bx\|_2^2 - 2\eta(\bx) \bx^\top \nabla^2 f(\bx)^{-1} \nabla f(\bx) + \eta(\bx)^2 \|\nabla^2 f(\bx)^{-1} \nabla f(\bx)\|_2^2\notag\\
    &\qquad +2\mu^2\|\bx\|_2 + 2\mu\eta(\bx)\|\nabla^2 f(\bx)^{-1}\nabla f(\bx)\|+ \mu^4 \label{eq:bound1}
\end{align}
where we used the fact that $\|\bvarepsilon\|_2\leq \mu^2$ 
and triangle inequality.
Plugging in the expression for $\nabla f(\bx)$, we have
\begin{align}
    \bx^\top \nabla^2 f(\bx)^{-1} \nabla f(\bx) &= -\dfrac{1}{n}\sum_{i=1}^n y_i p_i \bx^\top \nabla^2 f(\bx)^{-1} \ba_i^\top + \mu \bx^\top \nabla^2 f(\bx)^{-1} \bx\notag\\
    &\geq -\|\nabla^2f(\bx)^{-1} \bx\|_2 + \mu \bx^\top \nabla^2 f(\bx)^{-1} \bx\notag\\
    &\geq - \frac{1}{\mu} \|\bx\|_2 + \mu^2 \|\bx\|_2^2 \label{eq:bound2}
\end{align}
where the first inequality follows from \Cref{asm:bounded} and triangle inequality, and the second inequality is due to \Cref{cor:gradient-bounds}.
Similarly, we also have 
\begin{align*}
    \|\nabla^2f(\bx)^{-1} \nabla f(\bx)\|_2 &\leq \dfrac{1}{\mu} \|\nabla f(\bx)\|_2 \leq \frac{1}{\mu} + \|\bx\|_2.
\end{align*}
Combining this with \Cref{eq:bound1,eq:bound2}, we obtain (after rearrangement of the terms)
\begin{align*}
    \|\bx'\|_2^2 &\leq \|\bx\|_2^2 + \frac{2}{\mu} \eta(\bx)\|\bx\|_2 - 2\mu^2 \eta(\bx) \|\bx\|_2^2 + \eta(\bx)^2 \left(\frac{1}{\mu} + \|\bx\|_2\right)^2\\
    &\qquad + 2\mu^2\|\bx\|_2 + 2\mu\eta(\bx)\bigg(\frac{1}{\mu} + \|\bx\|_2\bigg) + \mu^4\\
    &= \|\bx\|_2^2 - 2\mu^2 \eta(\bx) \|\bx\|_2^2 + 2\mu^2\|\bx\|_2 + \mu^4\\
    &\qquad + \frac{2}{\mu} \eta(\bx)\|\bx\|_2 + \eta(\bx)^2 \left(\frac{1}{\mu} + \|\bx\|_2\right)^2 + 2\mu\eta(\bx)\left(\frac{1}{\mu} + \|\bx\|_2\right)
\end{align*}
Further applying the bounds on $\lambda(\bx)$ from \Cref{cor:gradient-bounds}, with direct computation, we have
\begin{align*}
    \|\bx'\|_2^2 &\leq \|\bx\|_2^2 \underbrace{- \frac{8\mu^3}{4\mu + 1+\mu\|\bx\|_2} \|\bx\|_2^2 + 2\mu^2\|\bx\|_2}_{\cI_1} + \mu^4\\
    &\qquad + \underbrace{\frac{4\sqrt{\mu(1+\mu)}\|\bx\|_2}{2\sqrt{\mu^3(1+\mu)} -\mu + \mu^2\|\bx\|_2} + \frac{4\mu(1+\mu)(1/\mu+\|\bx\|_2)^2}{(2\sqrt{\mu(1+\mu)} -1 + \mu\|\bx\|_2)^2} + \frac{8\mu(1/\mu + \|\bx\|_2)}{4\mu + 1+\mu\|\bx\|_2}}_{\cI_2}.
\end{align*}
Observe that for sufficiently large $\|\bx\|_2$, we have $\cI_1\approx -6\mu^2\|\bx\|_2$ 
and $\cI_2 = O(1)$.
This implies that $\|\bx'\|_2 \leq \|\bx\|_2$ for sufficiently large $\|\bx\|_2$, and thus completes the proof.
\end{proof}

\begin{lemma}\label{lem:norm_constraint}
Under \Cref{asm:bounded}, let $f$ be the regularized logistic loss defined in \Cref{eq:logistic_loss} with regularization parameter $\mu>0$.
Consider a sequence of iterates $\{\bx_t\}_{t\geq 0}$ satisfying 
\begin{align*}
    \bx_{t+1} = \bx_t - \frac{2\sqrt{\mu}}{2\sqrt{\mu}+\lambda(\bx_t)} \nabla^2 f(\bx_t)^{-1} \nabla f(\bx_t) + \bvarepsilon_t
\end{align*}
where $\|\bvarepsilon_t\|_2\leq \mu^2$ for all 
$t\geq 0$.
Then there exists a constant $C$ depending on $\mu$ such that $\|\bx_t\|_2\leq C$ for all $t\geq 0$.
\end{lemma}
\begin{proof}[Proof of \Cref{lem:norm_constraint}]
First, there exists a constant $C_1$ (given by \Cref{lem:norm_descent}) such that if $\|\bx_t\|_2\geq C_1$, then $\|\bx_{t+1}\|_2\leq \|\bx_t\|_2$.
Then we define 
\begin{align*}
    C_2 = \max_{\substack{\bx:\|\bx\|_2\leq C_1\\\bvarepsilon:\|\bvarepsilon\|_2\leq \mu^2}} \bigg\|\bx - \frac{2\sqrt{\mu}}{2\sqrt{\mu}+\lambda(\bx)} \nabla^2 f(\bx)^{-1} \nabla f(\bx) + \bvarepsilon\bigg\|_2. 
\end{align*}
Note that by \Cref{cor:gradient-bounds}, $C_2$ is a constant depending only on $C_1$ and the regularization parameter $\mu$.
Finally, we choose $C = \max\{\|\bx_0\|_2, C_1, C_2\}$, and the result follows.
\end{proof}

Our target is to prove convergence of the inexact damped Newton's method up to some error threshold, which depends on the bound of the error terms $\{\bvarepsilon_t\}$.
At a high level, the proof strategy is as follows:
\begin{itemize}[left=0pt]
    \item \textbf{Step 1}: Show by induction that if $\bx_t$ is feasible, then for sufficiently small error $\bvarepsilon_t$, so is the next iteration. This leads to some constraint on the error term $\bvarepsilon_t$.
    \item \textbf{Step 2}: Show constant decrease of the loss function when $\lambda(\hat{\bx}_t) \geq 1/6$.
    \item \textbf{Step 3}: Show that when $\lambda(\bx_t)<1/6$, we have $\lambda(\bx_{t+1}) \leq c\lambda_t^2 + \epsilon'$ for some constant $c>0$, so we enter the regime of quadratic convergence, which is maintained up to error $\varepsilon$.
\end{itemize}

\logisticonv*
\begin{remark}
In more detail the error obtained is $\alpha+\varepsilon$, where $\alpha \geq \sqrt{\epsilon(1+\mu)/(4\mu)}$, given that the algorithm performs $T = c+\log\log\dfrac{1}{1/2 + 3\alpha} + \log\log\dfrac{1}{\varepsilon}$, where $c$ denotes the initial steps of constant decrease of the function $g$. By picking $\varepsilon = \sqrt{\epsilon}$ we get the desired result. 
\end{remark}
\begin{proof} 
First by \cref{lem:norm_constraint}, we know that there exists a constant $C$ such that $\norm{\bx_t}_2\leq C$ for all $t\geq 0$.
We first derive an upper bound for $\delta_t := \|\bDelta_t\|_{\nabla^2 g(\bx_t)}$.
By definition, we have 
\begin{align*}
    \delta_t^2 &= (-\eta(\bx_t)\nabla^2 g(\bx_t)^{-1}\nabla g(\bx_t) + \bvarepsilon_t)^\top \nabla^2 g(\bx_t) (-\eta(\bx_t)\nabla^2 g(\bx_t)^{-1}\nabla g(\bx_t) + \bvarepsilon_t)\\
    &= \eta(\bx_t)^2 \nabla g(\bx_t)^\top \nabla^2 g(\bx_t)^{-1} \nabla g(\bx_t) - 2\eta(\bx_t) \bvarepsilon_t^\top \nabla g(\bx_t) + \bvarepsilon_t^\top \nabla^2 g(\bx_t)^{-1} \bvarepsilon_t\\
    &= \frac{\lambda(\bx_t)^2}{(1+\lambda(\bx_t))^2} - \frac{2\bvarepsilon_t^\top \nabla g(\bx_t)}{1+\lambda(\bx_t)} + \bvarepsilon_t^\top \nabla^2 g(\bx_t)^{-1} \bvarepsilon_t.
\end{align*}
It follows from \Cref{cor:gradient-bounds} that for all $t\geq 0$, $|\bvarepsilon_t^\top \nabla g(\bx_t)|\leq \epsilon (1+\mu\|\bx_t\|_2)/(4\mu)\leq \epsilon(1+C\mu)/(4\mu)$, and also $|\bvarepsilon_t^\top \nabla^2 g(\bx_t)^{-1} \bvarepsilon_t| \leq \epsilon^2 / 4$.
Therefore, we have
\begin{align}\label{eq:delta_t_bound}
    \delta_t &\leq \sqrt{\frac{\lambda(\bx_t)^2}{(1+\lambda(\bx_t))^2} + \frac{\epsilon(1+C\mu)}{2\mu} + \frac{\epsilon^2}{4}}\notag\\
    &\leq \frac{\lambda(\bx_t)}{1+\lambda(\bx_t)} + \sqrt{\frac{\epsilon(1+C\mu)}{2\mu} + \frac{\epsilon^2}{4}}.
\end{align}
Note that by \Cref{cor:gradient-bounds}, we have $\lambda(\bx_t)\leq (1+\mu\|\bx_t\|_2)/\sqrt{\mu}\leq (1+C\mu)/(2\sqrt{\mu})$.
Then there exists some constant $c_1$ depending only on $\mu$ such that $\delta_t<1$ when $\epsilon\leq c_1$. %

Now, we proceed to show that there is a constant decrease of the loss value up to the point that $\lambda(\bx)\leq 1/6$, after which we enter the regime of quadratic convergence. 

\paragraph{Phase I: Constant decrease of the loss function.}
Suppose $\lambda(\bx_t)\geq 1/6$ (note that if $\lambda(\bx_0)<1/6$, then we can directly proceed to the next phase).
Since $g$ is standard self-concordant, it follows from \Cref{thm:constant-decrement} that
\begin{align*}
    g(\bx_{t+1}) - g(\bx_t) &\leq \nabla g(\bx_t)^\top (\bx_{t+1}-\bx_t) + \omega_*(\|\bx_{t+1}-\bx_t\|_{\nabla^2 g(\bx_t)})\\
    &= -\eta(\bx_t)\nabla g(\bx_t)^\top \nabla^2 g(\bx_t)^{-1}\nabla g(\bx_t) + \bvarepsilon_t^\top \nabla g(\bx_t) + \omega_*(\|\bx_{t+1}-\bx_t\|_{\nabla^2 g(\bx_t)})\\
    &= - \frac{\lambda(\bx_t)^2}{1 + \lambda(\bx_t)} + \bvarepsilon_t^\top \nabla g(\bx_t) + \omega_*(\|\bx_{t+1}-\bx_t\|_{\nabla^2 g(\bx_t)})
\end{align*}
where we applied the definition of $\lambda(\bx_t)$ and $\eta(\bx_t)$ in the last equality.
Combining the above two equations, we have
\begin{align*}
    g(\bx_{t+1}) - g(\bx_t) &\leq - \frac{\lambda(\bx_t)^2}{1 + \lambda(\bx_t)} + \bvarepsilon_t^\top \nabla g(\bx_t) + \omega_*\bigg(\frac{\lambda(\bx_t)^2}{(1+\lambda(\bx_t))^2} - \frac{2\bvarepsilon_t^\top \nabla g(\bx_t)}{1+\lambda(\bx_t)} + \bvarepsilon_t^\top \nabla^2 g(\bx_t)^{-1} \bvarepsilon_t\bigg).
\end{align*}
Recall that $w_*(x) = -x-\log(1-x)$, and we view the right-hand side as a function of $\lambda(\bx_t)$ while regarding $c\equiv \bvarepsilon_t^\top \nabla g(\bx_t)$ and $c'\equiv\bvarepsilon_t^\top \nabla^2 g(\bx_t)^{-1} \bvarepsilon_t$ as constants, yielding 
\begin{align*}
    h(x) := - \frac{x^2}{1 + x} - \bigg(\frac{x^2}{(1+x)^2} - \frac{2c}{1+x} + c'\bigg) - \log\bigg(1 - \frac{x^2}{(1+x)^2} + \frac{2c}{1+x} - c'\bigg) + c.
\end{align*}
Since $|c|\leq 0.06$ by our assumption on $\bvarepsilon_t$, it can be verified that $h(x)$ is decreasing in $x$ for $x\geq 1/6$, and moreover $h(1/6) \leq 0.01$ by our assumption on $\bvarepsilon_t$ (see \cref{app:auxiliary-constant-decrease}).
Therefore, we have $g(\bx_{t+1}) - g(\bx_t) \leq -0.01$ for all $t$ such that $\lambda(\bx_t)\geq 1/6$.

\paragraph{Phase II: Quadratic convergence.}
Now suppose $\sqrt{\epsilon} < \lambda(\bx_t)<1/6$ (again, if $\lambda(\bx_0) < \sqrt{\epsilon}$, then we are done).
Note that there exists some constant $C_1$ such that $t\leq C_1$ 
due to the constant decrease of the loss function in the previous phase.
Then by \Cref{thm:convergence to optimum} and \Cref{lem:bound-on-omega}, we have 
\begin{align}\label{eq:suboptimality_bound}
    g(\bx_t) - g(\bx^*) \leq \frac{3\lambda(\bx_t)^2}{5}
\end{align}
where $\bx^*$ is the global minimizer of $g$.
Thus it suffices to characterize the decrease of $\lambda(\bx_t)$.

Applying \Cref{thm:cound-hessian}, we get
\begin{align}\label{eq:lambda_bound}
    \lambda(\bx_{t+1}) &= \|\nabla^2 g(\bx_{t+1})^{-1/2} \nabla g(\bx_{t+1})\|_2 \leq \frac{1}{1-\|\bDelta_t\|_{\nabla^2 g(\bx_t)}} \|\nabla^2 g(\bx_t)^{-1/2} \nabla g(\bx_{t+1})\|_2.
\end{align}
Note that $\nabla g(\bx_{t+1}) = \nabla g(\bx_t) + \int_0^1 \nabla^2 g(\bx_t + s\bDelta_t)\bDelta_t \diff s$.
Also, we have $\nabla g(\bx_t) = - (1+\lambda(\bx_t))\nabla^2 g(\bx_t) \bDelta_t + (1+\lambda(\bx_t)) \nabla^2 g(\bx_t) \bvarepsilon_t$.
Combining these two equations, we obtain 
\begin{align*}
    \nabla g(\bx_{t+1}) &= \underbrace{\int_0^1 \big(\nabla^2 g(\bx_t+s\bDelta_t) - (1+\lambda(\bx_t)\nabla^2 g(\bx_t))\big)\diff s}_{=: \bG_t} \cdot \bDelta_t  + (1+\lambda(\bx_t)) \nabla^2 g(\bx_t) \bvarepsilon_t
\end{align*}
where we introduced the notation $\bG_t$ for the integral term.
Therefore, by triangle inequality,
\begin{align}
    \|\nabla^2g(\bx_t)^{-1/2}\nabla g(\bx_{t+1})\|_2 &\leq \big\|\nabla^2 g(\bx_t)^{-1/2} \bG_t \bDelta_t \big\|_2 + (1+\lambda(\bx_t)) \big\|\nabla^2 g(\bx_t)^{1/2} \bvarepsilon_t\big\|_2\notag\\
    &= \big\|\eta(\bx_t)\nabla^2 g(\bx_t)^{-1/2} \bG_t \nabla^2 g(\bx_t)^{-1} \nabla g(\bx_t)\big\|_2\notag\\
    &\qquad + \big\|\nabla^2 g(\bx_t)^{-1/2} \bG_t \bvarepsilon_t\big\|_2 + (1+\lambda(\bx_t)) \big\|\nabla^2 g(\bx_t)^{1/2} \bvarepsilon_t\big\|_2 \label{eq:bound_hessian_inv_grad}
\end{align}
By \Cref{cor:bound-hessian}, it holds that 
\begin{align*}
    \bigg(-\delta_t + \frac{\delta_t^2}{3}-\lambda(\bx_t)\bigg)
    \nabla^2 g(\bx_t) \preceq \bG_t \preceq \bigg(\frac{1}{1-\delta_t} - 1 - \lambda(\bx_t)\bigg) \nabla^2 g(\bx_t).
\end{align*}
By direct computation, it can be verified that $0\leq \delta_t/(1-\delta_t)-\lambda(\bx_t) \leq \delta_t + \lambda(\bx_t)$.
This implies that $\|\nabla^2 g(\bx_t)^{-1/2}\bG_t\|_\op \leq (\delta_t + \lambda(\bx_t)) \|\nabla^2 g(\bx_t)^{-1/2}\|_\op$ and $\|\nabla^2 g(\bx_t)^{-1/2} \bG_t \nabla^2 g(\bx_t)^{-1/2}\|_\op \leq \delta_t + \lambda(\bx_t)$.
Applying these bounds to \Cref{eq:bound_hessian_inv_grad}, we obtain
\begin{align*}
    \|\nabla^2g(\bx_t)^{-1/2}\nabla g(\bx_{t+1})\|_2 &\leq \eta(\bx_t) (\delta_t+\lambda(\bx_t)) \|\nabla^2 g(\bx_t)^{-1/2} \nabla g(\bx_t)\|_2\\
    &\qquad + \epsilon (\delta_t+\lambda(\bx_t)) \|\nabla^2 g(\bx_t)^{-1/2}\|_\op + \epsilon(1+\lambda(\bx_t)) \|\nabla^2 g(\bx_t)^{1/2}\|_\op\\
    &\leq \frac{\lambda(\bx_t)(\delta_t+\lambda(\bx_t))}{1+\lambda(\bx_t)} + \epsilon(\delta_t + \lambda(\bx_t)) \bigg(\frac{1}{2} + \sqrt{\frac{1+\mu}{4\mu}}\bigg)
\end{align*}
where the last inequality follows from \Cref{cor:gradient-bounds}.
Plugging this into \Cref{eq:lambda_bound}, we obtain
\begin{align*}
    \lambda(\bx_{t+1}) &\leq \frac{\lambda(\bx_t)(\delta_t+\lambda(\bx_t))}{(1+\lambda(\bx_t))(1-\delta_t)} + \frac{\epsilon(\delta_t + \lambda(\bx_t))}{1-\delta_t} \bigg(\frac{1}{2} + \sqrt{\frac{1+\mu}{4\mu}}\bigg)\\
    &= \frac{\lambda(\bx_t)^2}{(1+\lambda(\bx_t))(1-\delta_t)}  + \frac{\delta_t\lambda(\bx_t)}{(1+\lambda(\bx_t))(1-\delta_t)} + \frac{\epsilon(\delta_t + \lambda(\bx_t))}{1-\delta_t} \bigg(\frac{1}{2} + \sqrt{\frac{1+\mu}{4\mu}}\bigg).
\end{align*}
Note that there exists some constant $c_2$ depending only on $\mu$ such that if $\epsilon<c_2$, then when $\lambda(\bx_t)<1/6$, we have $\delta_t\leq 1/5$ by \Cref{eq:delta_t_bound}.
Then, for some constant $C'>0$, we further have
\begin{align*}
    \lambda(\bx_{t+1}) &\leq 3\lambda(\bx_t)^2 + C' \epsilon.
\end{align*}
Let $\alpha = \sqrt{C'\epsilon/3} \leq 1/6$.
We can rewrite the above inequality as 
\begin{align*}
    \lambda(\bx_{t+1}) - \alpha &\leq 3\lambda(\bx_t)^2 + C'\epsilon + \alpha
    = 3(\lambda(\bx_t) - \alpha)^2.
\end{align*}
Telescoping this inequality, we obtain that for some constant $C_2>0$, $\lambda(\bx_t) \leq C''\sqrt{\epsilon}$ when $t\geq C_1 + C_2 \log\log\frac{1}{\epsilon}$.
Combining this with \eqref{eq:suboptimality_bound}, we see that for any such $t$, we have $g(\bx_t) - g(\bx^*) \leq \frac{3C''}{5} \epsilon$.
Finally, choosing $c=\min\{c_1,c_2\}$ completes the proof.
\end{proof}

\section{Transformers for logistic regression}\label{app:logistic}
In this section, we present the explicit construction of a linear Transformer that can emulate Newton's method on the regularized logistic loss, and we further provide its error analysis.

Recall the regularized logistic loss defined as
\begin{align*}
    f(\bx) & = \dfrac{1}{n}\sum_{i=1}^n\log(1+\exp(-y_i\bx^\top\ba_i)) + \dfrac{\mu}{2}\norm{\bx}_2^2,
\end{align*}
and its gradient and Hessian
\begin{align*}
    \nabla f(\bx) = -\dfrac{1}{n}\sum_{i=1}^n y_ip_i\ba_i + \mu \bx,
    \qquad\nabla^2 f(\bx) &= \dfrac{1}{n}\bA^\top \bD\bA  + \mu\id_d\succeq \mu \id_d
\end{align*}
where each $p_i = \exp(-y_i\bx^\top\ba_i)/(1+\exp(-y_i\bx^\top\ba_i))$ and $\bD = \diag(p_1(1-p_1), \ldots,p_n(1-p_n))$. 

Letting $g(\bx) \equiv f(\bx)/(4\mu)$, Newton's method on $g$ updates as follows:
\begin{align*}
\bx_{t+1} &= \bx_t - \dfrac{1}{1+\lambda_g(\bx_{t})}(\nabla^2 g(\bx_t))^{-1} \nabla g(\bx_t)\\
   &= \bx_t - \dfrac{2\sqrt{\mu}}{2\sqrt{\mu}+\lambda_f(\bx_{t})}(\nabla^2 f(\bx_t))^{-1} \nabla f(\bx_t)\\
    & = \bx_{t} - \dfrac{2\sqrt{\mu}}{2\sqrt{\mu}+\lambda_f(\bx_{t})}\bigg(\dfrac{1}{n}\bA^\top\bD\bA + \mu\id_t\bigg)^{-1} \bigg(\dfrac{1}{n}\sum_{i=1}^n y_ip_i\ba_i + \mu\bx_t\bigg)
\end{align*}
where $\lambda_g(\bx) = \sqrt{\nabla g(\bx)^\top (\nabla^2 g(\bx))^{-1} \nabla g(\bx)}$ and $\lambda_f(\bx) = \sqrt{\nabla f(\bx)^\top (\nabla^2 f(\bx))^{-1} \nabla f(\bx)}$.
To emulate the above update, we need to approximate the following components:
\begin{enumerate}
\item The values of the diagonal entries of the matrix $\bD$, resulting an error vector $ \bu_1\in\R^n$.
    
\item The multiplication of $\bD\bA$, which incurs an error $\bU_2\in\R^{n\times d}$.

\item The inversion of the Hessian matrix, which incurs an error $\bE_2\in\R^{d\times d}$.

\item The values of $\bp=(p_1,\ldots,p_n)^\top\in\RR^n$, which incurs an error $\bu_3\in\R^n$.
\item The step-size $2\sqrt{\mu}/(2\sqrt{\mu}+\lambda_f(\bx_t))$, which incurs an error $\epsilon_4$. 
Notice that to do this, we need to first approximate the value of $\lambda_f(\bx_t)$.
\end{enumerate}
For simplicity, we drop the subscript $f$ from $\lambda_f$, and we just write $\lambda$ from now on.
The resulting update for one step admits the following form:
\begin{equation}\label{eq:app-aproximate-updates}
\begin{aligned}
    \bx_1 = \bx_0 - \bigg(\dfrac{2\sqrt{\mu}}{2\sqrt{\mu}+\lambda(\bx_0)} + \epsilon_4\bigg)
   \bracks*{ \bigg(\dfrac{1}{n}\bA^\top\hat{\bD}\bA + \mu\id+ \bE_1\bigg)^{-1} + \bE_2}
    \bigg(\dfrac{1}{n} \sum_{i=1}^n y_i (\hat{p}_{i}-u_{3,i}) \ba_i +  \mu\bx_0\bigg)
\end{aligned}
\end{equation}
where $\bE_{1} = \bA^\top\text{diag}(\bu_{1})\bA + \bA^\top\bU_{2} $.

\paragraph{Input format.}
We consider the following input format:
\begin{align}
   \bH_0 = \begin{pmatrix}
         [\id_{d} \; \zero] \\
       [\id_{d} \; \zero] \\ 
          [\id_{d} \; \zero] \\ 
            \zero\\
         \bA^\top\\
         \by^\top\\
         \bx_0\one_n^\top\\
         \tfrac{1}{n}\be_1^\top\\
        \zero\\
        \one_n^\top
     \end{pmatrix}\in\R^{(7d+5)\times n}.
\end{align}
Here the first identity matrix will be used to store the updates for calculating the inverse, the second one for the initialization, while the third one for required computation. 
The initialization $\bx_0$ is copied $n$ times
The second to last line is used to store the parameter $\eta$, the step-size. 
Below we will provide explicit construction of Transformers that can realize the following map from input to output:
\begin{equation}\label{eq:one-step-logistic}
   \bH_0 \mapsto \bH_{K}= \begin{pmatrix}
 [\id_{d} \; \zero] \\
       [\id_{d} \; \zero] \\ 
          [\id_{d} \; \zero] \\ 
            \zero\\
         \bA^\top\\
         \by^\top\\
         \bx_1\one_n^\top\\
         \tfrac{1}{n}\be_1^\top\\
       \zero\\
        \one_n^\top
     \end{pmatrix}\in\R^{(7d+5)\times n}
\end{equation}
where $K$ denotes the number of Transformer layers. 

\constrlogistic*

For clarity, we split the proof into two parts: the first part is to construct the approximate updates of the Newton's algorithm, and the second part is devoted to the error analysis.

\begin{proof}[Proof of \Cref{thm:logistic}: Construction of the approximate updates]
For ease of presentation, in the following proof we omit the explicit expressions of the weights and describe instead the operations they induce.
The constructions are straightforward to be determined, similar to the proofs in \Cref{app:linear-reg-con}. Notice that each of the value, key and query weight matrices can always make any row selection by zero-padding and ignoring the rest of the rows of the input matrix. 

\paragraph{Step 1 - Create $\bd$.} 
We start by creating the diagonal entries of $\bD$. 
We choose $\bW_V, \bW_K, \bW_Q$ such that
\begin{equation}
    \bW_V\bH_0 = \begin{pmatrix}
        \zero\\
        \be_1^\top\\
        \zero
    \end{pmatrix}, \quad \bW_K\bH_0 = \bx_0\one_{n}^\top,\quad \bW_Q\bH_0 = \bA^\top, 
\end{equation}
Note that here the value weight matrix $\bW_V$ just picks the first row of the identity matrix, zeroes out any other row, and performs a row permutation to place the result in the second to last line. 
Then we have 
\begin{align}
    \bH^{\texttt{Attn}}_1 =\bH_0 + \begin{pmatrix}
    \zero\\
    \bx_0^\top\bA^\top\\
    \zero
    \end{pmatrix}
    =
    \begin{pmatrix}
        [\id_{d} \; \zero] \\
        [\id_{d} \; \zero] \\ 
        [\id_{d} \; \zero] \\ 
        \zero\\
        \bA^\top\\
        \by^\top\\
        \bx_0\one_n^\top\\
        \be_1^\top/n\\
        \bx_0^\top\bA^\top\\
        \one_n^\top
     \end{pmatrix}
\end{align}
we then use the ReLU network to approximate the values $d_i =D_{ii} = p(\bx_0^\top\ba_i,y_i)(1-p(\bx_0^\top\ba_i,y_i))$ where $p(x,y) = \exp(-yx)/(1+\exp(-yx))$. 
We zero out all the rows except for the rows of $\by^\top, \bx_0^\top\bA^\top$ and construct the weights of the ReLU network to approximate the function $p(x,y)(1-p(x,y))$. 
Notice that this function takes values between $[0,1]$ and it domain contains $x\in\R$ and $y\in\{-1,1\}$ in this case. 
Also note that it is symmetric with respect to the value $y$.
Therefore, it suffices to approximate the function $f(x) = e^{x}/(1+e^{x})^2$, which is increasing for $x\leq0$ and decreasing for $x\geq 0$. 
We approximate it by splitting $[0,1]$ into $1/\epsilon$ intervals and deal with each of these intervals seperately for $x\geq 0$ and $x\leq 0$ using a total number of $4/\epsilon$ ReLU neurons. 
This gives rise to
\begin{align*}
    \bH_1 = \begin{pmatrix}
    [\id_{d} \; \zero] \\
    [\id_{d} \; \zero] \\ 
    [\id_{d} \; \zero] \\ 
    \zero\\
    \bA^\top\\
    \by^\top\\
    \bx_0\one_n^\top\\
    \be_1^\top/n\\
    \hat{\bd}^\top\\
    \one_n^\top
    \end{pmatrix}
\end{align*}
where $\hat{\bd}=\bd +\bu_1$. 
Notice that 
\begin{equation}\label{eq:first-error}
    \norm{\bu_1}\leq \dfrac{4}{N}
\end{equation}
where $N$ is the width of the ReLU layer.

\paragraph{Step 2 - Approximate $\dfrac{1}{n}\bD\odot\bA^\top$.} 
In the attention of the second layer, we compute $\hat{\bd}/n$ by setting
\begin{equation}
    \bW_V^{(1)}\bH_1 = \begin{pmatrix}
        \zero\\
        \be_1^\top\\
        \zero
    \end{pmatrix}, \quad
    \bW_K^{(1)}\bH_1 = \dfrac{1}{n}\be_1^\top, \quad
    \bW_Q^{(1)}\bH_1 = \hat{\bd}, 
\end{equation}
and we use one more head to subtract the residual by letting
\begin{equation}
    \bW_V^{(2)}\bH_1 = \begin{pmatrix}
        \zero\\
        -\be_1^\top\\
        \zero
    \end{pmatrix}, \quad
    \bW_K^{(2)}\bH_1 =\be_1^\top, \quad
    \bW_Q^{(2)}\bH_1 = \hat{\bd}.
\end{equation}
Thus, 
\begin{align*}
    \bH_2^{\texttt{Attn}} = \bH_1 + \begin{pmatrix}
        \zero\\
        \hat{\bd}/n\\
        \zero
     \end{pmatrix} - \begin{pmatrix}
         \zero\\
         \hat{\bd}\\
         \zero
     \end{pmatrix}
     =\begin{pmatrix}
        [\id_{d} \; \zero] \\
        [\id_{d} \; \zero] \\
        [\id_{d} \; \zero] \\ \zero\\
        \bA^\top\\
        \by^\top\\
        \bx_0\one_n^\top\\
        \be_1^\top/n\\
       \hat{\bd}/n\\
       \one_n^\top
    \end{pmatrix}.
\end{align*}
The next ReLU layer approximates the multiplication of the diagonal matrix $\frac{1}{n}\hat{\bD}$ with the matrix $\bA$, which is the same as creating the vectors $\hat{d}_1 \ba_1/n, \hat{d}_2\ba_2/n,\hdots \hat{d}_n\ba_n/n$, or equivalently, the matrix $\frac{1}{n}\hat{\bD}\odot\bA^\top$. 
This can be implemented with one ReLU layer since the elements will be processed serially and the information needed is on the same column. 
This approximation incurs an error matrix $\bU_2$. 
This yields the output of the second Transformer layer:
\begin{equation}
    \bH_2 = \begin{pmatrix}
        \one_d\hat{\bd}^\top\odot\bA^\top/n +\bU_2 \\
        [\id_{d} \; \zero] \\ 
        [\id_{d} \; \zero] \\ 
        \zero\\
        \bA^\top\\
        \by^\top\\
        \bx_0\one_n^\top\\
        \be_1^\top/n\\
        \zero_n^\top\\
        \one_n^\top
    \end{pmatrix}.
\end{equation}
Notice that we can easily subtract the identity matrix since we have another copy of it. 
Using Proposition A.1 in \cite{bai2023transformers}, we have that $s=2$ and each $\hat{d}_i\in[-0.1,1.1]$ 
thus it requires a total number of $O(1/(n^2\epsilon^2))$ ReLUs to achieve $\epsilon$ accuracy for each entry. 
Moreover, 
\begin{align*}
\one_d \dfrac{\hat{\bd}^\top\odot\bA^\top}{n} + \bU_2 = \one_d \dfrac{\bd^\top\odot\bA^\top}{n} +\one_d\dfrac{\bu_1^\top\odot\bA^\top}{n}+ \bU_2,
\end{align*}
and thus 
\begin{equation}\label{eq:second-error}
    \twonorm{\bU_2}\leq \fnorm{\bU_2} \leq \widetilde{\mathcal{O}}(\sqrt{dn}\epsilon)  = \widetilde{ \mathcal{O}}\bigg(\dfrac{\sqrt{d}}{\sqrt{nN}}\bigg)
\end{equation}
where $N$ is the width of the ReLU layer and we omit logarithmic factors in $\widetilde \cO(\cdot)$.

\paragraph{Step 3 - Create the matrix $\tfrac{1}{n}\bA^\top\bD\bA + \mu\id$.}
For this step, we first use the attention layer to implement
\begin{align*}
    \bW_V^{(1)}\bH_2 = \begin{pmatrix}
        \alpha\bA^\top\\
        \bA^\top\\
        \zero
        \end{pmatrix}, \quad
    \bW_K^{(1)}\bH_2 = \dfrac{1}{n}\one_d^\top\hat{\bd}\odot\bA^\top + \bU_2, \quad
    \bW_{Q}^{(1)}\bH_2 = [\id_d \; \zero].
\end{align*}
Notice that $(\bW_K^{(1)}\bH_2)^\top = \frac{1}{n}\hat{\bD}\bA + \bU_2^\top$. 
We also use two extra heads such that
\begin{align*}
    \bW^{(2)}_V\bH_2 = \begin{pmatrix}
        \alpha\mu[\id_d\;\zero]\\
       (\mu-1) [\id_d\;\zero]\\
        \zero
    \end{pmatrix}, \quad 
    \bW_K^{(2)}\bH_2 = [\id_d \; \zero], \quad
    \bW_Q^{(2)}\bH_2 =[\id_d \;\zero], \\
    \bW^{(3)}_V\bH_2 = \begin{pmatrix}
        [\id_d\;\zero]\\
        \zero
    \end{pmatrix}, \quad 
    \bW_K^{(3)}\bH_2 = [\id_d \; \zero], \quad
    \bW_Q^{(3)}\bH_2 =-\one_d\hat{\bd}^\top\odot\bA^\top/n -\bU_2.
\end{align*}
Combining these three heads, we get
\begin{align*}
    \bH_3 &= \bH_2 + \begin{pmatrix}
        [\alpha(\frac{1}{n} \bA^\top\hat{\bD}\bA + \bA^\top\bU_2^\top)\; \zero]\\
        [\frac{1}{n}\bA^\top\hat{\bD}\bA + \bA^\top\bU_2^\top\;\zero]\\
        \zero
    \end{pmatrix}+ \begin{pmatrix}
        [\alpha\mu\id_d\;\zero]\\
        [(\mu-1)\id_d\;\zero]\\
        \zero
    \end{pmatrix} + \begin{pmatrix}
        -\frac{\one_d}{n}\hat{\bd}^\top\odot\bA^\top -\bU_2\\
        \zero
    \end{pmatrix}\\
    &=\begin{pmatrix}
        [\alpha(\frac{1}{n} \bA^\top\hat{\bD}\bA+\bA^\top\bU_2^\top+ \mu\id_d); \zero]\\
        [\frac{1}{n} \bA^\top\hat{\bD}\bA+\bA^\top\bU_2^\top+\mu\id_d; \zero]\\
          [\id_{d} \; \zero] \\ 
          \zero\\
         \bA^\top\\
         \by^\top\\
         \bx_0\one_n^\top\\
            \be_1^\top/n\\
       \zero_n^\top\\
       \one_n^\top
    \end{pmatrix}.
\end{align*}

For simplicity, we denote $\bB := \frac{1}{n} \bA^\top\hat{\bD}\bA+\bA^\top\bU_2^\top+ \mu\id_d$.
We then use one more attention layer with two head as follows:
For the first head, 
\begin{equation}
    \bW_V^{(1)}\bH_3 = \begin{pmatrix}
        [\id_d\;\zero]\\
        \zero
    \end{pmatrix}, \quad
    \bW_K^{(1)}\bH_3 = \alpha\bB, \quad
    \bW_Q^{(1)}\bH_3 = [\id_d\;\zero].
\end{equation}
The second head is used to subtract the residual, so that the output of this layer is
\begin{align*}
    \bH_4 = \begin{pmatrix}
        [\alpha\bB^\top; \zero]\\
        [\bB; \zero]\\
        [\id_{d} \; \zero] \\ 
          \zero\\
         \bA^\top\\
         \by^\top\\
         \bx_0\one_n^\top\\
            \be_1^\top/n\\
       \zero_n^\top\\
       \one_n^\top
    \end{pmatrix}.
\end{align*}
We keep track of the errors in the following way:
\begin{align}\label{eq:error-hessian}
    \bB = \dfrac{1}{n}\bA^\top\bD\bA +\mu\id +\dfrac{1}{n}\bA^\top \mathrm{diag}(\bu_1)\bA + \bA^\top\bU_2
    = \dfrac{1}{n}\bA^\top\bD\bA + \mu\id +\bE_1
\end{align}
where $\bE_1$ can be controlled as
\begin{equation}\label{eq:errorforhessian}
\begin{aligned}
    \twonorm{\bE_1} &\leq \dfrac{1}{n}\twonorm{\diag(\bu_1)}\twonorm{\bA}^2 + \twonorm{\bA}\twonorm{\bU_2}\\
    &\leq \dfrac{1}{n}\norm{\bu_1}_2 \fornorm{\bA}^2+\fornorm{\bA}\twonorm{\bU_2}\\
    &\leq \norm{\bu_1}_2 + \sqrt{n}\twonorm{\bU_2}
\end{aligned}
\end{equation}
since $\norm{\ba_i}^2 \leq 1$ for all $i=1,\hdots,n$ by \Cref{asm:bounded}.

\paragraph{Step 4 - Invert the matrix.} 
Next, we implement Newton's iteration as in the previous section for $k$ %
steps and we get 
\begin{align*}
    \bH_{2k+4} =  \begin{pmatrix}
        [\bB^{-1} +\bE_2\;\zero]\\
         [\id_d \;\zero] \\
          [\id_{d} \; \zero] \\
          \zero\\
         \bA^\top\\
         \by^\top\\
         \bx_0\one_n^\top\\
          \be_1^\top/n\\
       \zero_n^\top\\
       \one_n^\top
    \end{pmatrix}
\end{align*}
\paragraph{Step 5 - Create the $\{p_i\}_{i=1}^n$.} 
We repeat Step 1 to get
\begin{equation}
    \bH^{\texttt{Attn}}_{2k+5} = \begin{pmatrix}
        [\bB^{-1}+\bE_2\;\zero]\\
         [\id_d \;\zero] \\
          [\id_{d} \; \zero] \\ 
            \zero\\
         \bA^\top\\
         \by^\top\\
         \bx_0\one_n^\top\\
          \be_1^\top/n\\
       \bx_0^\top\bA^\top\\
       \one_n^\top
    \end{pmatrix}.
\end{equation}
We then use the ReLU layer to approximate $\{p_i\}_{i=1}^n$ and we have
\begin{equation}
    \bH_{2k+5} = \begin{pmatrix}
        [\bB^{-1}+\bE_2\;\zero]\\
         [\id_d \;\zero] \\
          [\id_{d} \; \zero] \\
          \zero\\
         \bA^\top\\
         \by^\top\\
         \bx_0\one_n^\top\\
          \be_1^\top/n\\
       \bp^\top +\bu_3^\top\\
       \one_n^\top
    \end{pmatrix}.
\end{equation}
The function approximation here is the same as the one in Step 1, and it requires $2/\epsilon$ ReLUs to achieve $\epsilon$ accuracy.
Then
\begin{equation}\label{eq:third-error}
    \norm{\bu_3} \leq \dfrac{2}{N}
\end{equation}
where $N$ is the width of the ReLU layer.

\paragraph{Step 6 - Calculate the values $\{y_ip_i/n\}_{i=1}^n$.} 
In the attention layer, we multiply each $p_i$ with $1/n$ as we did in Step 2,
and we have 
\begin{align*}
    \bH_{2k+6}^{\texttt{Attn}} = \begin{pmatrix}
        [\bB^{-1}+\bE_2\;\zero]\\
         [\id_d \;\zero] \\
          [\id_{d} \; \zero] \\ 
          \zero\\
         \bA^\top\\
         \by^\top\\
         \bx_0\one_n^\top\\
          \be_1^\top/n\\
       (\bp^\top +\bu_3^\top)/n\\
       \one_n^\top
    \end{pmatrix}.
\end{align*}
We then use the ReLU layer to approximate the inner product between $\frac{1}{n}\bp$ and $\by$.
To do so, notice that each $p_i\in(0,1)$, and $y=\{-1,1\}$ and consider the following sets of ReLUs:
\begin{align*}
   o_1 = \sigma\bigg(\dfrac{1}{2}x +2y\bigg) - \sigma\bigg(-\dfrac{1}{2}x + 2y\bigg),\quad
   o_2 = \sigma\bigg(-\dfrac{1}{2}x -2y\bigg) - \sigma\bigg(+\dfrac{1}{2}x - 2y\bigg)
\end{align*}
Suppose $x\in(0,1)$, then if $y=1$, the outputs are $o_1 = x$, $o_2 = 0$, and for $y=-1$, the outputs are $o_1 = 0$ and $o_2 = -x$, thus $o_1 + o_2 = xy$. 
Consequently,
\begin{align*}
    \bH_{2k+6} = \begin{pmatrix}
        [\bB^{-1} +\bE_2\;\zero]\\
         [\id_d \;\zero] \\
          [\id_{d} \; \zero] \\ 
          \zero\\
         \bA^\top\\
         \by^\top\\
         \bx_0\one_n^\top\\
          \be_1^\top/n\\
       (\bp^\top +\bu_3^\top)\odot\by^\top/n\\
       \one_n^\top
    \end{pmatrix}
\end{align*}

\paragraph{Step 7 - Calculate the gradient.} 
Next, we want to calculate the quantity $-\frac{1}{n}\sum_{i=1}^ny_ip_i\ba_i +\mu\bx_0$. We first set the weight matrices of the attention layer such that
\begin{align*}
    \bW^{(1)}_V\bH_{2k+6} = \begin{pmatrix}
        \zero\\
        -\bA^\top\\
        \zero
    \end{pmatrix}, \quad
    \bW^{(1)}_K\bH_{2k+6} = \dfrac{1}{n}(\bp^\top +\bu_4^\top)\odot\by^\top, \quad
    \bW^{(1)}_Q\bH_{2k+6} = \be_1^\top\\
    \bW^{(2)}_V\bH_{2k+6} = \begin{pmatrix}
        \zero\\
        -[\id_d\;\zero]\\
        \zero
    \end{pmatrix}, \quad
    \bW^{(2)}_K\bH_{2k+6} = [\id_d\;\zero], \quad
    \bW^{(2)}_Q\bH_{2k+6} = [\id_d\;\zero],
\end{align*}
and we need a third head to add $\mu\bx_0$ by setting
\begin{equation}
    \bW^{(3)}_V\bH_{2k+6} = \begin{pmatrix}
        \zero\\
        \bx_0\one^\top_n\\
        \zero
    \end{pmatrix}, \quad
    \bW^{(3)}_K\bH_{2k+6} = \be_1, \quad
    \bW^{(3)}_Q \bH_{2k+6}= \mu\be_1,
\end{equation}
Thus, we place the result in the second block of the matrix and we have 
\begin{equation}
    \bH_{2k+7} = \begin{pmatrix}
        [\bB^{-1}+\bE_2\;\zero]\\
       [ \bb\; \zero]\\
          [\id_{d} \; \zero] \\ 
          \zero\\
         \bA^\top\\
         \by^\top\\
         \bx_0\one_n^\top\\
          \be_1^\top/n\\
       (\bp^\top +\bu_3^\top)\odot\by^\top/n\\
       \one_n^\top
    \end{pmatrix}
\end{equation}
where $\bE_2$ is the error incurred by running Newton's method for the inversion of the matrix and $\bb = -\frac{1}{n}\sum_{i=1}^n(y_ip_i + y_iu_{3i})\ba_i +\mu\bx_0 = -\frac{1}{n}\sum_{i=1}^n y_ip_i\ba_i + \mu\bx_0 +\bm{\epsilon}_3$, $\bm{\epsilon}_3 = -\frac{1}{n}\sum_{i=1}^ny_iu_{3i}\ba_i $.  Thus,
\begin{equation}\label{eq:error-third-final}
    \norm{\bm{\epsilon}_3}_2\leq \dfrac{2}{N}
\end{equation}
where is the width of the ReLU layer used to approximate the error $\bu_3$ (\cref{eq:third-error}).

\paragraph{Step 8 - Calculate the stepsize.} 
We proceed to approximate $\hat{\lambda}(\bx)^2 = \nabla^\top f(\bx)(\nabla^2f(\bx))^{-1}\nabla f(\bx) = \bb^\top(\bB^{-1}+\bE_2)\bb$. 
We first create the quantity $(\bB^{-1}+\bE_2)\bb$ which we also need for the update, store it, and then calculate in the next layer the parameter $\lambda(\bx)^2$.  
For the first layer we set
\begin{align*}
    \bW_V^{(1)}\bH_{2k+7} = \begin{pmatrix}
        [\bB^{-1}+\bE_2\;\zero]\\
        \zero
    \end{pmatrix},\quad 
    \bW_K^{(1)}\bH_{2k+7} = [\id_d\;\zero], \quad
    \bW_Q^{(1)}\bH_{2k+7} = [\bb\;\zero],\\
    \bW_V^{(2)}\bH_{2k+7} = -\begin{pmatrix}
        [\bB^{-1}+\bU_3\;\zero]\\
        \zero
    \end{pmatrix}, \quad
    \bW_K^{(2)}\bH_{2k+7} = [\id_d\;\zero], \quad
    \bW_Q^{(2)}\bH_{2k+7} =  [\id_d\;\zero].
\end{align*}
Then we get
\begin{align*}
    \bH_{2k+8} = \begin{pmatrix}
        [(\bB^{-1}+\bE_2)\bb \;\zero]\\
       [ \bb\; \zero] \\
          [\id_{d} \; \zero] \\
          \zero\\
         \bA^\top\\
         \by^\top\\
         \bx_0\one_n^\top\\
          \be_1^\top/n\\
       (\bp^\top+\bu_3^\top)\odot\by^\top/n\\
       \one_n^\top
     \end{pmatrix}.
\end{align*}
We use another Transformer layer to calculate the quantity $\lambda(\hat{\bx}_t)^2$. 
We set the attention layer as follows
\begin{align*}
    \bW_V^{(1)}\bH_{2k+8} = \begin{pmatrix}
        \zero\\
           \be_1^\top\\
           \zero
       \end{pmatrix},\quad
    \bW_K^{(1)}\bH_{2k+8} = [\bb \; \zero], \quad
    \bW_Q^{(1)}\bH_{2k+8} = [(\bB^{-1}+\bE_2)\bb\;\zero]
\end{align*}
where we place the $\be_1^\top$ in the second to last position.
We also use as before an extra head to remove the residual. 
Then we get
\begin{align*}
    \bH_{2k+9}^{\texttt{Attn}} = \begin{pmatrix}
       [(\bB^{-1}+\bE_2)\bb\;\zero]\\
       [ \bb \; \zero] \\
          [\id_{d} \; \zero] \\ 
          \zero\\
         \bA^\top\\
         \by^\top\\
         \bx_0\one_n^\top\\
          \be_1^\top/n\\
       [\hat{\lambda}(\hat{\bx}_t)^2\;\zero]\\
       \one_n^\top
     \end{pmatrix}.
\end{align*}
In the ReLU layer%
, we approximate the function $\frac{2\sqrt{\mu}}{2\sqrt{\mu}+\sqrt{x}}$. %
Notice that this function take values in $(0,1]$.
The first derivative of the function is $-\frac{2\sqrt{\mu}}{2(2\sqrt{\mu}+\sqrt{x})^2\sqrt{x}}<0$, so it is a monotonically decreasing function. 
Thus, using the same argument with previous steps we can approximate it up to error $\epsilon$ with $2/\epsilon$ ReLUs. 
This yields an error $\epsilon_4$
\begin{equation}\label{eq:fourth-error}
    \abs{\epsilon_4 }\leq \dfrac{2}{N} 
\end{equation}
where $N$ is the width of the ReLU layer and we can write the stepsize as $\hat{\eta}(\bx_0) = 2\sqrt{\mu}/(2\sqrt{\mu} + \lambda(\bx_0) )+ \epsilon_4$.
Thus, 
\begin{align*}
    \bH_{2k+9} = \begin{pmatrix}
       [(\bB^{-1}+\bE_2)\bb\;\zero]\\
       [ \bb \; \zero] \\ 
          [\id_{d} \; \zero] \\ \zero\\
         \bA^\top\\
         \by^\top\\
         \hat{\bx}_t\one_n^\top\\
 \be_1^\top/n\\
      [\hat{\eta}(\bx_0)\; *]\\
       \one_n^\top
     \end{pmatrix}
\end{align*}
where $*$ denotes inconsequential values. 

\paragraph{Step 9 - Update} 
For the last step, we use two more attention layers. 
We set the weights for the first layer such that
\begin{align*}
    \bW_V^{(1)}\bH_{2k+9} = \begin{pmatrix}
        \zero\\
            [\hat{\eta}(\bx_0)\; *]\\
           \zero
       \end{pmatrix}, \quad
    \bW_K^{(1)}\bH_{2k+9} = [(\bB^{-1}+\bE_2)\bb\;\zero],\quad
     \bW_Q^{(1)}\bH_{2k+9} = [\id_d\;\zero].
\end{align*}
We also use an extra head to remove the residual by setting
\begin{align}
    \bW_V^{(2)}\bH_{2k+9} = -\begin{pmatrix}
        \zero\\
        [\hat{\eta}(\bx_0)\; *]\\
        \zero
    \end{pmatrix},\quad
    \bW_K^{(2)}\bH_{2k+9} = [\id_d\;\zero], \quad
    \bW_Q^{(2)}\bH_{2k+9} = [\id_d\;\zero], 
\end{align}
where the values $[\hat{\eta}(\bx_0)\; *]$ are placed in the second to last row. 
Combining these two heads, we get
\begin{align*}
    \bH_{2k+10} &= \bH_{2k+9} + \begin{pmatrix}
        \zero\\
       [ \hat{\eta}(\bx_0)\bb^\top(\bB^{-1}+\bE_2)^\top \;\zero]\\
       \zero
    \end{pmatrix} - \begin{pmatrix}\zero\\
   [\hat{\eta}(\bx_0)\; *]\\
    \zero      
    \end{pmatrix}\\
    &=\begin{pmatrix}
       [ (\bB^{-1}+\bE_2)\bb\;\zero]\\
       [ \bb \; \zero]  \\
          [\id_{d} \; \zero] \\ \zero\\
         \bA^\top\\
         \by^\top\\
         \bx_0\one_n^\top\\
          \be_1^\top/n\\
       [\hat{\eta}(\bx_0)\bb^\top(\bB^{-1}+\bE_2)^\top\;*]\\
       \one_n^\top
     \end{pmatrix}
\end{align*}
where $*$ are again inconsequential values. 
Notice that when subtracting the residual with the second head, we correct only the first $d$ columns. 

Finally, we perform the update with one more attention layer:
\begin{align*}
    \bW_V^{(1)}\bH_{2k+10} = \begin{pmatrix}
        \zero\\
            -[\id_d\;\zero]\\
           \zero
       \end{pmatrix},\quad
    \bW_K^{(1)}\bH_{2k+10} = [\hat{\eta}(\bx_0) \bb^\top(\bB^{-1}+\bE_2)^\top\;*], \quad
    \bW_Q^{(1)}\bH_{2k+10} =  \one_n^\top.
\end{align*}
Another head is used to restore the matrix in its initial form:
\begin{align*}
    \bW_V^{(2)}\bH_{2k+10} = \begin{pmatrix}
        [ -\id_d\;\zero \;\id_d \; \zero]\\
        [\zero;-\id_d\;\id_d \; \zero]\\
        \zero
    \end{pmatrix}\bH_{2k+10},\quad
    \bW_K^{(2)}\bH_{2k+10} = [\id_d\;\zero], \quad
    \bW_Q^{(2)}\bH_{2k+10} = [\id_d\;\zero].
\end{align*}
As a result, we get
\begin{align}
    \bH_{2k+11}^{\texttt{Attn}} &= %
\begin{pmatrix}
       [ \id_d\;\zero]\\
       [ \id_d\; \zero]  \\
          [\id_{d} \; \zero] \\ \zero\\
         \bA^\top\\
         \by^\top\\
         (\hat{\bx}_t - \hat{\eta}(\bx_0)(\bB^{-1}+\bE_2)\bb)\one_n^\top\\
          \be_1^\top/n\\
      [ \hat{\eta}(\bx_0)\bb^\top(\bB^{-1}+\bE_2)^\top\;*]\\       
       \one_n^\top
     \end{pmatrix}.
\end{align}
We further use the ReLUs to zero out the second to last row. 
Note that the inconsequential values wer create when approximating the function $1/(1+\sqrt{x})$, so they are close to one. 
Thus, we can use the following simple ReLU to zero out this line:
\begin{align*}
    \sigma(-x/2 + 5y) - \sigma(x/2 + 5y)
\end{align*}
where we use as $x$ the elements of the second to last row and $y$, the last row and all the other connections are zero.
The final output is
\begin{align*}
     \bH_{2k+11}= \begin{pmatrix}
       [ \id_d\;\zero]\\
       [ \id_d\; \zero]  \\
          [\id_{d} \; \zero] \\ \zero\\
         \bA^\top\\
         \by^\top\\
         (\bx_0 -\hat{\eta}(\bx_0)(\bB^{-1}+\bE_2)\bb )\one_n^\top\\
          \be_1^\top/n\\
       \zero_n^\top\\       
       \one_n^\top
     \end{pmatrix} = \begin{pmatrix}
       [ \id_d\;\zero]\\
       [ \id_d\; \zero]  \\
          [\id_{d} \; \zero] \\ \zero\\
         \bA^\top\\
         \by^\top\\
         \bx_1\one_n^\top\\
          \be_1^\top/n\\
       \zero_n^\top\\       
       \one_n^\top
     \end{pmatrix}.
\end{align*}
Thus, for the next step the input is in the correct form and we can repeat the steps above for the next iteration.
\end{proof}

Next, we present the error analysis of the emulated iterates.
The target is to show that the updates described in \cref{eq:app-aproximate-updates} can be recasted to the following updates:
\begin{align*}
    \hat{\bx}_{t+1} = \hat{\bx}_t - \eta(\hat{\bx}_t)(\nabla^2f(\hat{\bx}_t))^{-1} \nabla f(\hat{\bx}_t) + \bvarepsilon_t
\end{align*}
where the error term $\bvarepsilon_t$ is controlled to satisfy the considitions of \cref{thm:convergence-of-approximate-updates}. 

\begin{proof}[Proof of \Cref{thm:logistic}: Error analysis]
For $\bx_0$, let $C$ be the constant given by \cref{lem:norm_constraint}.
From the previous step of weight constructions, we have the following update
\begin{align}\label{eq:approx_update_error}
     \bx_1 = \bx_0 - \bigg(\dfrac{2\sqrt{\mu}}{2\sqrt{\mu}+\hat{\lambda}(\bx_0)} + \epsilon_{4}\bigg)\bracks{(\nabla^2f(\bx_0)+ \bE_1 )^{-1} +\bE_2}(\nabla f(\bx_0)+\bm{\epsilon}_3)
\end{align}
where
\begin{align*}
    \bE_1 &= \dfrac{1}{n}\bA^\top \mathrm{diag}(\bu_1)\bA + \bA^\top\bU_2\\
    \bE_2 &= \text{Error of Newton's method for inversion}\\
    \bm{\epsilon}_3 &= -\dfrac{1}{n}\sum_{i=1}^ny_iu_{4i}\ba_i\\
    \epsilon_4 &=\text{Approximation error for the quantity $\lambda$.}
\end{align*}
where the error terms admit the following bounds
\begin{align}\label{eq:bounds-all-together}
    \norm{\bu_1} \leq \dfrac{4}{N},\quad
    \norm{\bU_2}_F \leq \widetilde{\mathcal{O}}\bigg(\dfrac{\sqrt{d}}{\sqrt{nN}}\bigg),\quad
    \norm{\bE_2} \leq \epsilon_2,\quad
    \norm{\bm{\epsilon}_3} \leq \dfrac{2}{N}, \quad
    \abs{\epsilon_4}\leq \dfrac{2}{N}
 \end{align}
where $N$ is the width. These bounds were derived in \cref{eq:first-error,eq:second-error,eq:third-error,eq:fourth-error} for $\bu_1, \bU_2 $ $\bm{\epsilon}_3$ and $\epsilon_4$ respectively. The error term $\epsilon_2$ is controlled by the number of layers, for $k = c + 2\log\log(1/\epsilon)$ layers we achieve error less than $\epsilon$.

Applying \Cref{cor:inverse}, we have $(\nabla^2 f(\bx_0) + \bE_1)^{-1} = \nabla^2 f(\bx_0)^{-1} + \bE_1'$ where $\bE_1'$ satisfies $\twonorm{\bE_1'}\leq \twonorm{\bE_1}/(\mu(\mu-\twonorm{\bE_1}))$.
Further writing $\bE_2' := \bE_1' +\bE_2$, then by \eqref{eq:bounds-all-together}, it holds that
\begin{equation}\label{eq:errorin-bE_2'}
    \twonorm{\bE_2'} \leq \twonorm{\bE_2} + \frac{\twonorm{\bE_1}}{\mu(\mu-\twonorm{\bE_1})}.
\end{equation}
Now we can rewrite \eqref{eq:approx_update_error} as
\begin{align*}
    \bx_1 = \bx_0 - \bigg(\dfrac{2\sqrt{\mu}}{2\sqrt{\mu}+\hat\lambda(\bx_0)} + \epsilon_{4}\bigg) \bracks{(\nabla^2f(\bx_0))^{-1} +\bE'_2 } (\nabla f(\bx_0)+\bm{\epsilon}_3)
\end{align*}

Next we analyze the error involved in the quantity $\hat{\lambda}^2(\hat{\bx}_t)$.
Recall that
\begin{align*}
    \hat{\lambda}(\bx_0)^2 = \bb^\top\bB\bb
    &= (\nabla f(\bx_0) + \bm{\epsilon}_3)^\top\bracks{(\nabla^2f(\bx_0))^{-1} +\bE'_2}(\nabla f(\bx_0) +\bm{\epsilon}_3)\\
    &= \underbrace{\nabla f(\bx_0)^\top (\nabla^2f(\bx_0))^{-1}\nabla f(\bx_0)}_{\lambda(\bx_0)^2}\\
    &\qquad + \underbrace{2\bepsilon_3^\top (\nabla^2 f(\bx_0))^{-1} (\nabla f(\bx_0) + \bepsilon_3) + (\nabla f(\bx_0) + \bm{\epsilon}_3)^\top \bE'_2 (\nabla f(\bx_0) +\bm{\epsilon}_3)}_{=:\epsilon_4'}
\end{align*}
By triangle inequality and the bounds from \Cref{cor:gradient-bounds}, we can bound $\epsilon_4'$ as follows:
\begin{equation}\label{eq:error-in-epsilon_4'}
    \abs{{\epsilon}_4'}< \dfrac{2(1+\mu C)\norm{\bm{\epsilon}_3} +\norm{\bm{\epsilon}_3}^2}{\mu} +\norm{\bm{\epsilon}_3}^2\twonorm{\bE'_2}+ 2(1+\mu C)\norm{\bm{\epsilon}_3}\twonorm{\bE'_2} +(1+\mu C)^2\twonorm{\bE_2'}
\end{equation}
Now, as long as it holds that 
\begin{align}\label{eq:bounds-1-for-errors}
    \twonorm{\bE_2'} = \mathcal{O}\parens*{\dfrac{\mu\epsilon_4^2}{(1+C\mu)^2}} \text{ and }
    \norm{\bm{\epsilon}_3} = \mathcal{O}\parens*{\dfrac{\epsilon_4^2\mu^2}{(1+C\mu)}}
\end{align}
we would have
\begin{equation}\label{eq:boundonbound}\dfrac{2(1+C\mu)\norm{\bm{\epsilon}_3} +\norm{\bm{\epsilon}_3}^2}{\mu} +\norm{\bm{\epsilon}_3}^2\twonorm{\bE'_2}+ 2(1+C\mu)\norm{\bm{\epsilon}_3}\twonorm{\bE'_2} +(1+C\mu)^2\twonorm{\bE_2'} < \dfrac{4\mu\epsilon_4^2}{1-\epsilon_4}.
\end{equation}
Then by \cref{app:auxiliary-control-error} we have that if the condition above holds
\begin{align*}
    \abs{\dfrac{2\sqrt{\mu}}{2\sqrt{\mu}+\hat{\lambda} (\bx_0)} - \dfrac{2\sqrt{\mu}}{2\sqrt{\mu}+\lambda(\bx_0)}}\leq {\epsilon_4}
\end{align*}
This implies that 
\begin{align*}
    \dfrac{2\sqrt{\mu}}{2\sqrt{\mu}+\hat{\lambda} (\bx_0)} =  \dfrac{2\sqrt{\mu}}{2\sqrt{\mu}+\lambda(\bx_0)} + \tilde \epsilon_4
\end{align*}
for some $\tilde\epsilon_4$, such that $\abs{\tilde\epsilon_4}\leq {\epsilon_4} $, since we also account for the approximation error from the ReLUs. 
Then it follows that
\begin{align*}
   \bx_1 &= \bx_0 - \bigg(\dfrac{2\sqrt{\mu}}{2\sqrt{\mu}+\lambda(\bx_0)} + \tilde\epsilon_{4}\bigg)\bracks{(\nabla^2f(\bx_0))^{-1} +\bE'_2 }(\nabla f(\bx_0)+\bm{\epsilon}_3)\\
   &= \bx_0 - \dfrac{2\sqrt{\mu}}{2\sqrt{\mu}+\lambda(\bx_0)}(\nabla^2f(\bx_0))^{-1}\nabla f(\bx_0)  - \tilde\epsilon_4\bracks{(\nabla^2f(\bx_0))^{-1} +\bE'_2 }(\nabla f(\bx_0)+\bm{\epsilon}_3)\\
   &\qquad- \dfrac{2\sqrt{\mu}}{2\sqrt{\mu}+\lambda(\bx_0)}\bE'_2(\nabla f({\bx}_0)+\bm{\epsilon}_3)\\
   &= {\bx}_0 - \dfrac{2\sqrt{\mu}}{2\sqrt{\mu}+\lambda({\bx}_0)}(\nabla^2f({\bx}_0))^{-1}\nabla f({\bx}_0) + \bvarepsilon
\end{align*}
where the error term $\bvarepsilon$ satisfies
\begin{equation}\label{eq:final-error-iteration}
    \|\bvarepsilon\|_2 \leq |\tilde\epsilon_4| \bigg(\dfrac{1}{\mu} +\twonorm{\bE_2'}\bigg) (1 + C\mu +\norm{\bm{\epsilon}_3}_2) + \twonorm{\bE_2'}(1+ C\mu +\norm{\bm{\epsilon}_3}).
\end{equation}
Using \cref{eq:bounds-1-for-errors} we have that 
\begin{align*}
    \|\bvarepsilon\|_2 &\leq |\tilde\epsilon_4| \bigg(\dfrac{1}{\mu} +\twonorm{\bE_2'}\bigg) (1+C\mu +\twonorm{\bm{\epsilon}_3}) + \twonorm{\bE_2'}(1+C\mu+\norm{\bm{\epsilon}_3})\\
    &\leq\dfrac{2\epsilon_4}{\mu}(1+C\mu+\twonorm{\bm{\epsilon}_3}) + (1+\epsilon_4)\twonorm{\bE_2'} (1+C\mu +\twonorm{\bm{\epsilon}_3})
\end{align*}
Notice now if
\begin{equation}
    \twonorm{\bE_2'} = \mathcal{O}\parens*{\dfrac{\epsilon^2\mu^3}{(1+C\mu)^4}} \text{ and } \twonorm{\bm{\epsilon}_3} = \mathcal{O}\parens*{\dfrac{\epsilon^2\mu^4}{(1+C\mu)^3}},
\end{equation}
then 
\begin{align}\label{eq:bound_epsilon_4}
    \epsilon_4 = \mathcal{O}\bigg(\dfrac{\mu\epsilon}{1+C\mu+C\mu^3}\bigg) = \mathcal{O}\parens*{\dfrac{\mu\epsilon}{1+C\mu}},
\end{align}
and consequently, $\twonorm{\bvarepsilon}\leq \epsilon$.

We will now study how we can bound the error term $\bE_2'$. Notice that from \cref{eq:errorforhessian,eq:errorin-bE_2'}, the following conditions hold:
\begin{align} 
    \twonorm{\bE_1} \leq \twonorm{\bu_1} + \sqrt{n}\twonorm{\bU_2}, \quad
    \twonorm{\bE_2'} \leq \twonorm{\bE_2} + \frac{\twonorm{\bE_1}}{\mu(\mu -\twonorm{\bE_1})}.
\end{align}
Now, given that
\begin{align}\label{eq:conditions-for-E1-E2}
    \twonorm{\bE_1}  = \mathcal{O}\bigg(\frac{\epsilon^2\mu^5}{(1+C\mu)^4}\bigg),\quad
    \twonorm{\bE_2} = \mathcal{O}\bigg(\frac{\epsilon^2\mu^3}{(1+C\mu)^4}\bigg),
\end{align}
it is straightforward to verify that indeed $\bE_2'$ satisfies \cref{eq:bounds-1-for-errors}. 
For the bound on $\bE_1$, it suffices to have 
\begin{equation}\label{eq:conditions-for-E1}
    \twonorm{\bu_1}  = \mathcal{O}\parens*{\dfrac{\epsilon^2\mu^5}{(1+C\mu)^4}},\quad
    \twonorm{\bU_2} = \mathcal{O}\parens*{\dfrac{\epsilon^2\mu^5}{\sqrt{n}(1+C\mu)^4}}
\end{equation}

It remains to determine the necessary width and depth for these error bounds to be achieved.
We combine the bounds from \cref{eq:bounds-all-together,eq:conditions-for-E1,eq:bound_epsilon_4,eq:conditions-for-E1-E2} to get that 
\begin{align*}
    \norm{\bu_1} \sim \dfrac{1}{N} & \Rightarrow N\sim \dfrac{(1+C\mu)^4}{\epsilon^2\mu^5}\\
    \norm{\bU_2}_F \sim \dfrac{\sqrt{d}}{\sqrt{nN}} &\Rightarrow N \sim \dfrac{d(1+C\mu)^8}{\epsilon^4\mu^{10}}\\
     \norm{\bm{\epsilon}_3} \sim \dfrac{1}{N} &\Rightarrow N\sim \dfrac{(1+C\mu)^3}{\epsilon^2\mu^4}\\
     \abs{\epsilon_4} \sim \dfrac{1}{N} &\Rightarrow N\sim \dfrac{1+C\mu}{\epsilon\mu}
\end{align*}
Finally, the error term $\twonorm{\bE_2}$ is controlled by the depth of the network and scales as  $k\sim 2\log\kappa(\frac{1}{n}\bA^\top\hat{\bD}\bA + \mu\id+ \bE_1) + \log\log\dfrac{(1+C\mu)^3}{\epsilon\mu^2}$, where the condition number of $\dfrac{1}{n}\bA^\top\hat{\bD}\bA + \mu\id+ \bE_1$ should be close to the condition number $\max_{\bx}\kappa(\nabla^2 f(\bx))$ ( For a formal proof of this statement see \cref{app:auxiliary-condition}).
This completes the proof.

\end{proof}

\subsection{Main Result}\label{app:main-result} We are now able to state our main result with explicit bounds. 

\maintheorem*
\begin{remark}
    Both $T$ and $k$ are actually have just one extra additive constant  term. In the first case, depends on the number of steps needed for $\lambda(\bx)$ to become less than $1/6$, while for $k$ the constant term is $2\ceil{\log\kappa(\nabla^2f(\bx))}$.
 \end{remark}
 
\begin{proof}
The proof follows by combining \cref{thm:convergence-of-approximate-updates} and \cref{thm:logistic}.  
\end{proof}

\begin{section}{Experiments}\label{appendix:exp}

During training, we utilize the Adam optimizer, with a carefully planned learning rate schedule. The task is learned through curriculum learning~\cite{garg2022can}, where we begin training with a small problem dimension and gradually increase the difficulty of the tasks. When implementing the Transformer backbone, we adhere to the structure of NanoGPT2\footnote{https://github.com/karpathy/nanoGPT/blob/master/model.py}, modifying only the causal attention to full attention.

For both the task of linear and logistic regression, we sample the data points as follows: We sample a random matrix $\bA$ and create its SVD decomposition, \ie $\bA = \bU\bS\bV^\top$. We then sample the maximum eigenvalue $\lambda_{max}$ uniformly random in the range of $[1,100]$. We then set the minimum eigenvalue as $\lambda_{min} = \lambda_{max}/ \kappa$, where $\kappa$ is the condition number of the problem at hand and it is fixed. Finally, we sample the rest of the eigenvalues uniformly random between $[\lambda_{min},\lambda_{max}]$ and recreate the matrix $\bS$, with the new eigenvalues.  We then create our covariance matrix as 
$\Sigma = \bU\bS\bU^\top$. We use $\Sigma$,  to sample the data samples $\bx_i$ from a multivariate Gaussian with mean $0$ and covariance matrix $\Sigma$.

In Figures~\ref{fig:lsa-full} and ~\ref{fig:lsa-ln-full}, we present comprehensive results for the LSA and LSA with layernorm models, trained on linear regression tasks across various condition numbers and noise levels. The trained LSA and LSA with layernorm consistently find higher-order methods when attempting to solve ill-conditioned linear regression problems.

\begin{figure}
    \centering
    \includegraphics[width=0.18\linewidth]{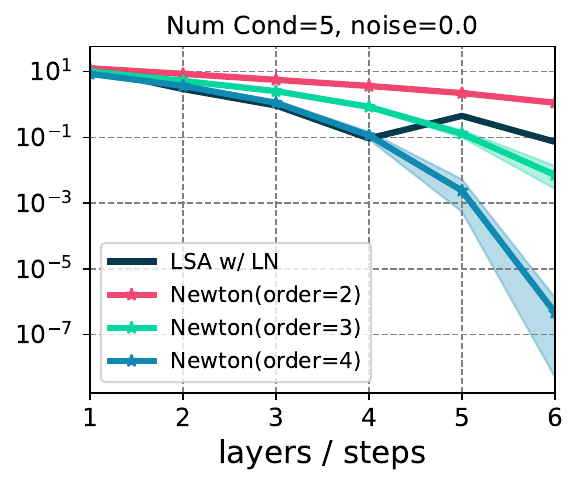}
    \includegraphics[width=0.18\linewidth]{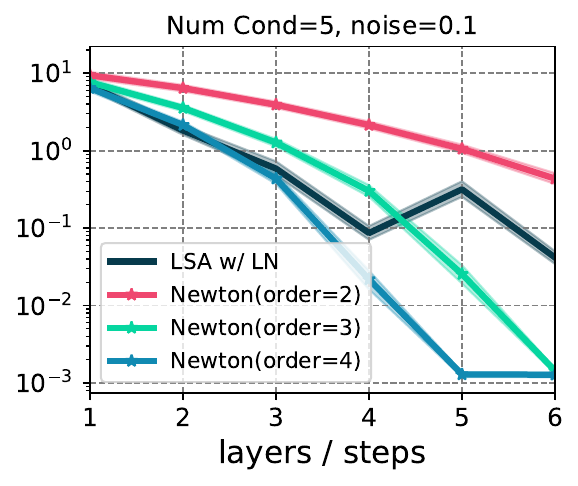}
    \includegraphics[width=0.18\linewidth]{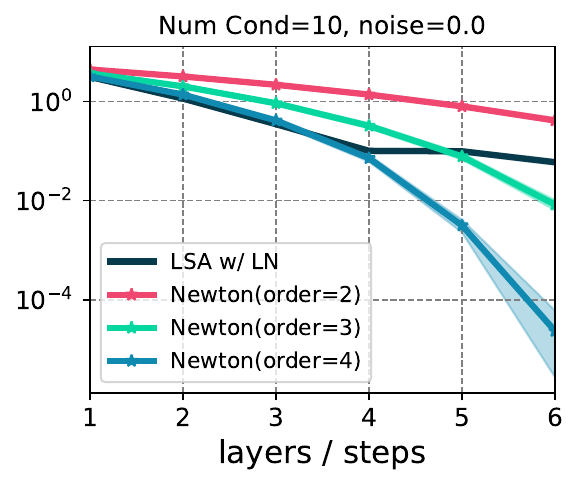}
    \includegraphics[width=0.18\linewidth]{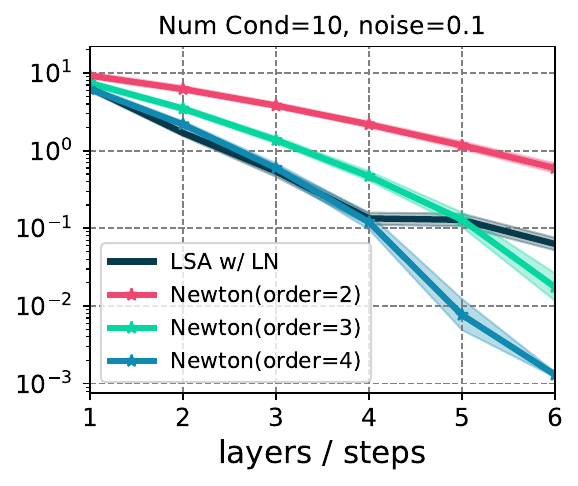}
    \caption{\small \emph{Loss of LSA and different order Newton iteration for linear regression with different input conditions.} }
    \label{fig:lsa-full}
\end{figure}

\begin{figure}
    \centering
    \includegraphics[width=0.18\linewidth]{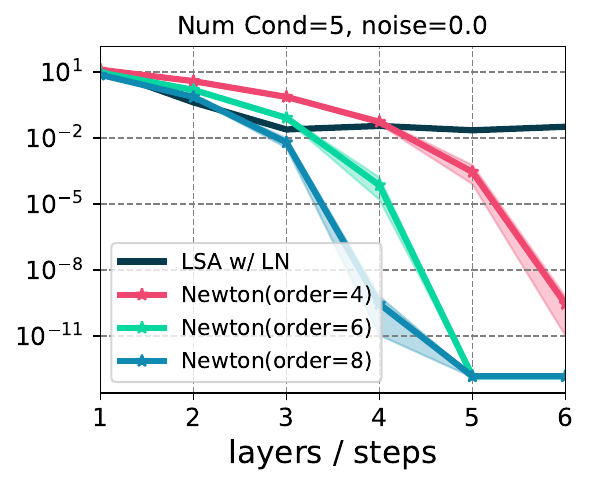}
    \includegraphics[width=0.18\linewidth]{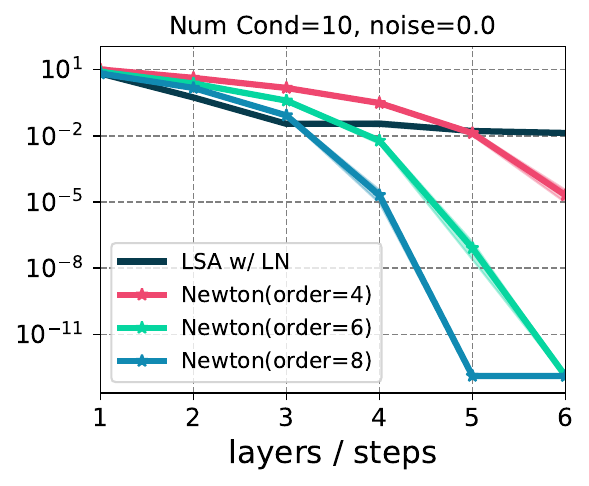}
    \includegraphics[width=0.18\linewidth]{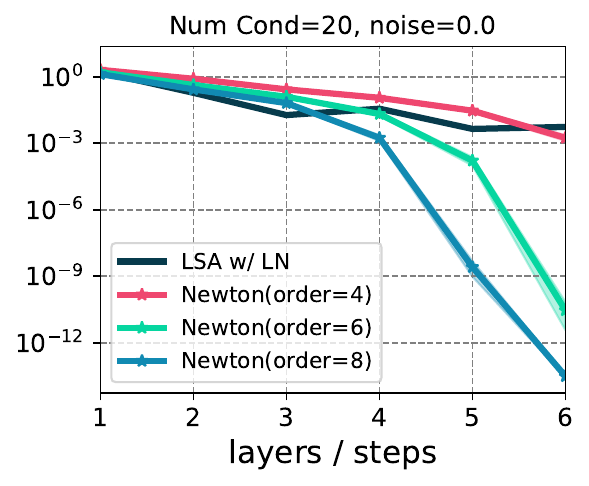}
    \includegraphics[width=0.18\linewidth]{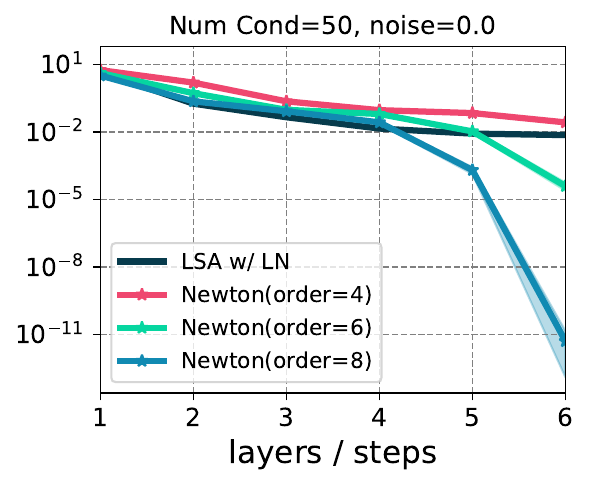}
    \includegraphics[width=0.18\linewidth]{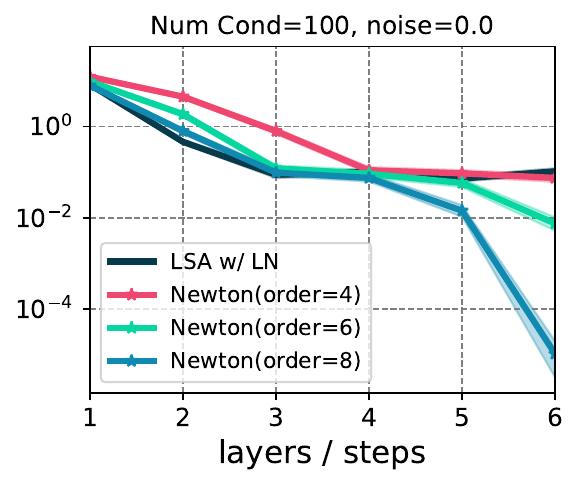}
    \includegraphics[width=0.18\linewidth]{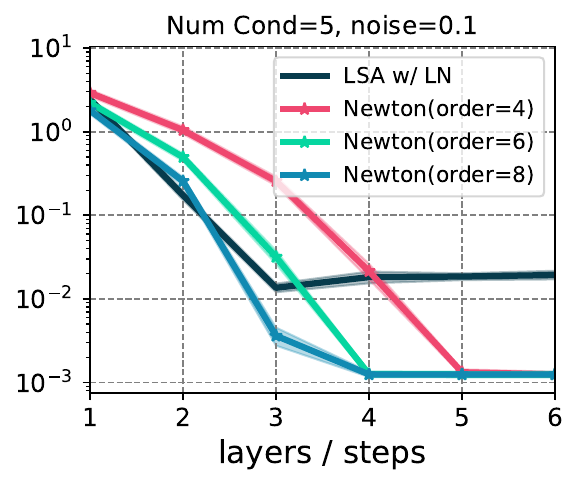}
    \includegraphics[width=0.18\linewidth]{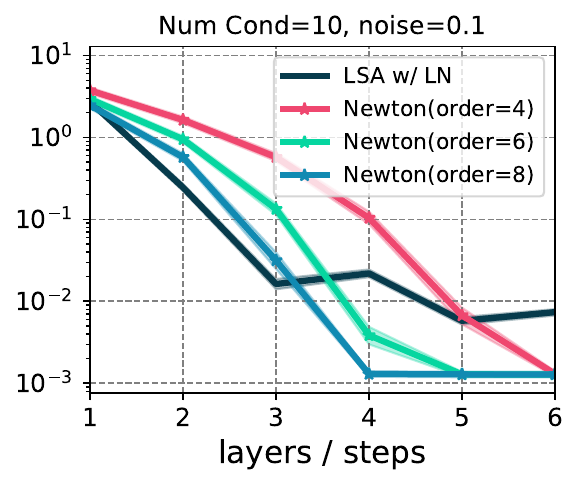}
    \includegraphics[width=0.18\linewidth]{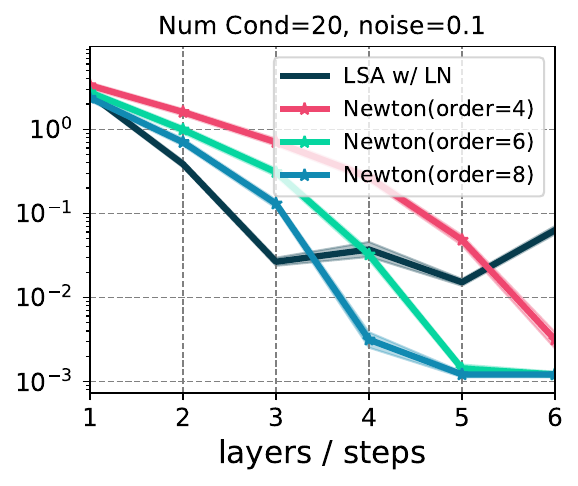}
    \includegraphics[width=0.18\linewidth]{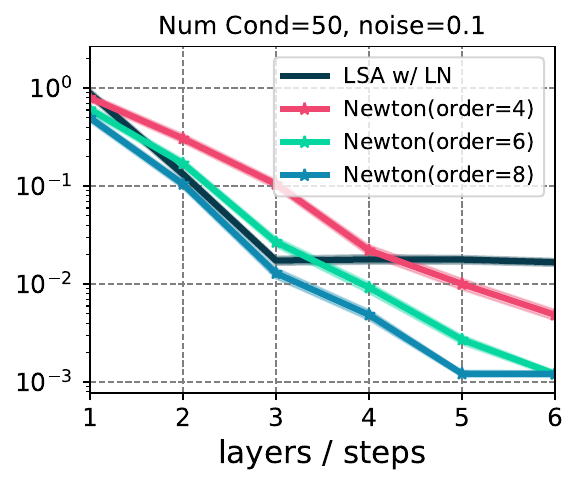}
    \includegraphics[width=0.18\linewidth]{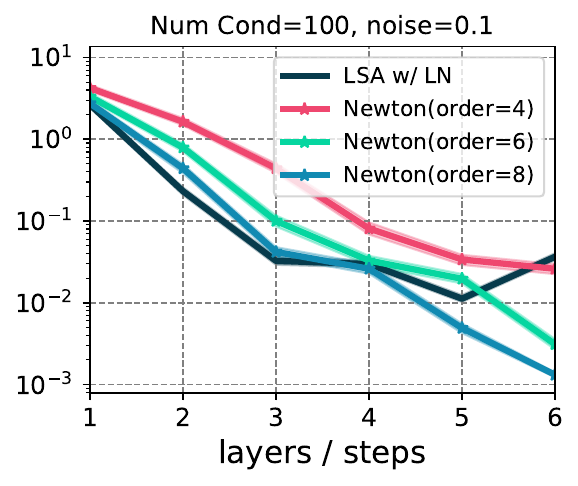}
    \caption{\small \emph{Loss of LSA w/ layernorm and different order Newton iteration for linear regression with different input conditions.} }
    \label{fig:lsa-ln-full}
\end{figure}

\end{section}
\section{Auxiliary results}\label{app:auxiliary}
We collect here some auxiliary results that are used in the proofs of the main results.
\subsection{Auxiliary result for constant decrease of the logistic loss}\label{app:auxiliary-constant-decrease}
Let 
\begin{align*}
    h(x) = -\dfrac{x^2}{1+x} +c -log(1-\tilde{\delta}) -\tilde{\delta}
\end{align*}
where $\tilde{\delta} = \sqrt{x^2/(1+x)^2 -2c/(1+x) +c'}$ and $c,c'$ are constants with respect to $x$. Mainly, we see the RHS of the previous inequality as a function of $\lambda$ and consider $\bE_t^\top\nabla g(\hat{\bx}_t),\bE_t^\top\nabla^2g(\hat{\bx}_t)\bE_t$ as constants. Then we have that 
\begin{align*}
    h'(x) &= -\dfrac{x^2 + 2x}{(1+x)^2} + \dfrac{1}{1-\tilde{\delta}}\dfrac{\mathrm{d}\tilde{\delta}}{\mathrm{d}x} -\dfrac{\mathrm{d}\tilde{\delta}}{\mathrm{d}x}\\
    &=-\dfrac{x^2 + 2x}{(1+x)^2} + \dfrac{\tilde{\delta}}{1-\tilde{\delta}}\dfrac{\mathrm{d}\tilde{\delta}}{\mathrm{d}x}\\
    &= -\dfrac{x^2 + 2x}{(1+x)^2} + \dfrac{\tilde{\delta}}{1-\tilde{\delta}}\dfrac{x+c(1+x)}{\tilde{\delta}(1+x)^3}\\
    &= -\dfrac{(x^2 + 2x)(1+x)(1-\tilde{\delta})}{(1+x)^3(1-\tilde{\delta})}+\dfrac{x+c(1+x)}{(1-\tilde{\delta})(1+x)^3}\\
    & = \dfrac{x+c(1+x) - (x^2 + 2x)(1+x)(1-\tilde{\delta})}{(1-\tilde{\delta})(1+x)^3}
\end{align*}
We use mathematica to plot this function (see \cref{fig:constant-decrease}) and the max value it can attain, when $x\in[1/6,1]$ and $\abs{c} \leq 0.06$ and we have that the maximum value of $g'$  is approximately $-0.02$. This implies that $h$ is decreasing, since we have that $\abs{c}\leq \norm{\bE_t}\norm{\nabla f(\hat{\bx}_t)}\leq \dfrac{\epsilon(1+\mu)}{4\mu} \leq 0.06$. Thus, we have
\begin{align*}
    g(\hat{\bx}_{t+1}) - g(\hat{\bx}_t) &\leq h(1/6)\\
    &= -\dfrac{1}{42} + y - \log(1-\sqrt{1/49 -12y/7 + z} -\sqrt{1/49 -12y/7 + z}
\end{align*}
where $y = \bE_t^\top\nabla g(\hat{\bx}_t)$ and $z = \bE_t^\top\nabla^2g(\hat{\bx}_t)\bE_t$. Notice now that
\begin{equation*}
    \abs{z} \leq \dfrac{\epsilon^2(1+\mu)}{4\mu} \leq 0.01^2 \text{ and }\abs{y} \leq 0.01
\end{equation*}
Since $\epsilon \leq \min\braces{0.01, 0.04\mu/(1+\mu)}$. Given these bounds we have that 
\begin{equation*}
    g(\hat{\bx}_{t+1}) - g(\hat{\bx}_t)\leq  -\dfrac{1}{42} + y -  \log(1-\sqrt{1/49 -12y/7 + 0.01^2}-\sqrt{1/49 -12y/7} 
\end{equation*}
We again use mathematica and plot this function for $\abs{y} \leq 0.012$, which can be viewed in \cref{fig:constant-decrease} and we see that we get a constant decrease of at least $0.01$.
\begin{figure}
    \centering
    \includegraphics[scale = 0.47]{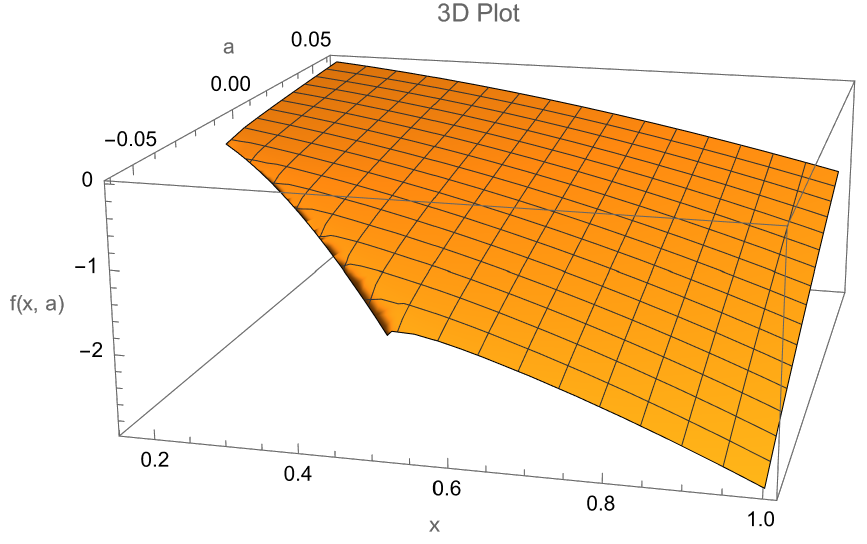}
    \hspace{3em}
    \includegraphics[scale=0.47]{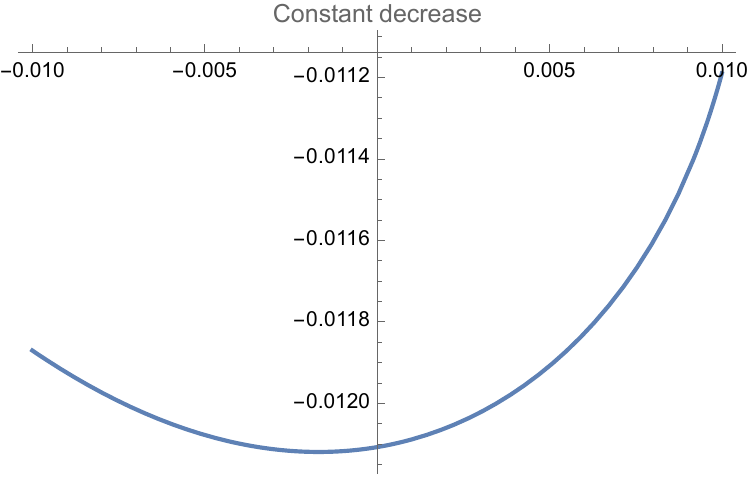}
    \caption{Left: The derivative of $h$ as a function of both $x$ and $c$ for $x\in[1/6,1]$ and $\abs{c}\leq 0.06$. Right: We see that the function is decreasing at least $-0.01$ at each step.}
    \label{fig:constant-decrease}
\end{figure}
\subsection{Auxiliary result for controlling the error}\label{app:auxiliary-control-error}

To control the change that this quantity can evoke, we note that we have approximated the function $g(x) = 2\sqrt{\mu}/(2\sqrt{\mu}+\sqrt{x})$ by discretizing $(0,1]$, so whenever $x\leq\alpha^2$, $g(x) \geq 2\sqrt{\mu}/(2\sqrt{\mu}+\alpha)$. 
Thus, if $x\leq 4\mu\epsilon_4^2/(1-\epsilon_4)^2$ we have that $g(x)\geq 1-\epsilon_4 $, similarly if $x\geq 4\mu(1-\epsilon_4)^2/\epsilon_4^2$ we have that $g(x)\leq \epsilon_4$. 
Now notice that given $x,\tx$ such that $\max\braces{\tx,x}\geq 4\mu\epsilon_4^2/(1-\epsilon_4)^2$ (otherwise we have already covered the case)  we have
\begin{align*}
    \abs{g(x) - g(\tx) } &=\abs{\dfrac{2\sqrt{\mu}}{2\sqrt{\mu} + \sqrt{x}} -\dfrac{2\sqrt{\mu}}{2\sqrt{\mu} + \sqrt{x'}}}\\
    &\leq 2\sqrt{\mu}\abs{\dfrac{\sqrt{x} -\sqrt{x'}}{(2\sqrt{\mu} + \sqrt{x})(2\sqrt{\mu} + \sqrt{x'})}}\\
    &\leq 2\sqrt{\mu} \abs{\dfrac{x-x'}{4\mu(\sqrt{x}+\sqrt{x'})}}\\
    &\leq 2\sqrt{\mu} \abs{\dfrac{x-x'}{4\mu\sqrt{\max\braces{\tx,x}}}}\\
    &\leq \dfrac{\abs{x-x'}(1-\epsilon_4)}{4\mu\epsilon_4}.
\end{align*}
Thus, if it holds that
\begin{equation}\label{eq:bound-for-lambda}
    \abs{x-x'} \leq 4\mu\dfrac{\epsilon_4^2}{(1-\epsilon_4)},
\end{equation}
then the function is always less than $\epsilon_4$. 

\subsection{Perturbation bounds}
\begin{theorem}[Corollary 2.7, p. 119 in \cite{Stewart1990MatrixPT}]\label{thm:errorinverse}
    Let $\kappa(\bA)= \norm{\bA}_2 \norm{\bA^{-1}}_2$ be the condition number of $\bA$. If $\tilde{\bA} = \bA +\bE$ is non-singular, then 
\begin{align*}
    \norm{\tilde{\bA}^{-1} -\bA^{-1}}_2 \leq \kappa(\bA)\dfrac{\norm{\bE}_2\norm{\tilde{\bA}^{-1}}_2}{\norm{\bA}_2}.
\end{align*}
If in addition $\kappa(\bA)\dfrac{\twonorm{\bE}}{\twonorm{\bA}}<1$ then 
\begin{align*}
    \twonorm{\tilde{\bA}^{-1}} \leq \dfrac{\twonorm{\bA^{-1}}}{1- \kappa(\bA)\dfrac{\twonorm{\bE}}{\twonorm{\bA}}},
\end{align*}
and thus
\begin{align*}
    \twonorm{\tilde{\bA}^{-1} -\bA^{-1}} \leq \dfrac{\kappa(\bA)\dfrac{\twonorm{\bE}}{\twonorm{\bA}}}{1 -\kappa(\bA)\dfrac{\twonorm{\bE}}{\twonorm{\bA}}}\twonorm{\bA^{-1}}.
\end{align*}
\end{theorem}

\begin{corollary}\label{cor:inverse}
Let $f$ be the regularized logistic loss defined in \eqref{eq:logistic_loss} with regularization parameter $\mu>0$.
For the matrix $\bB = (\nabla^2 f(\bx) + \bE_1)$, it holds that
\begin{align*}
    \twonorm{\bB^{-1} - (\nabla^2f(\bx))^{-1}} \leq \dfrac{\twonorm{\bE_1}}{\mu(\mu-\twonorm{\bE_1})}
\end{align*}
or equivalently, $\bB^{-1} = (\nabla^2f(\bx))^{-1} + \bE'_1 $ with $\twonorm{\bE'_1}\leq \dfrac{\twonorm{\bE_1}}{\mu(\mu-\twonorm{\bE_1})}$.
\end{corollary}
\begin{proof}
From \cref{thm:errorinverse} we have that 
\begin{align*}
    \twonorm{\bB^{-1} - (\nabla^2f(\bx))^{-1}} &\leq \dfrac{\twonorm{(\nabla^2f(\bx))^{-1}}^2\twonorm{\bE_1}}{1- \twonorm{(\nabla^2f(\bx))^{-1}}\twonorm{\bE_1}}
    \leq \dfrac{\twonorm{\bE_1}}{\mu(\mu-\twonorm{\bE_1})}
\end{align*}
because $\twonorm{(\nabla^2f(\bx))^{-1}}\leq 1/\mu$ and we have assumed that $\twonorm{\bE_1} \leq \mu$.
\end{proof}

\subsection{Condition number of perturbed matrix}\label{app:auxiliary-condition}

To show that the condition number of the matrix $\bB = \nabla f(\bx) + \bE$ is close to the condition number of $\nabla f(\bx)$ we will use Weyl's inequality.
\begin{lemma}[Weyl's Corollary 4.9 in \cite{Stewart1990MatrixPT}]
Let $\lambda_i$ be the eigenvalues of a matrix $\bA$ with $\lambda_1 \geq \hdots \geq \lambda_n$, $\tilde{\lambda}_i$ be the eigenvalues of a perturbed matrix $\tilde{\bA} = \bA + \bE$ and finally let $\epsilon_1 \geq \hdots \epsilon_m$ be the eigenvalues of $\bE$.    For $i=1,\hdots,n$ it holds that
    \begin{equation}
        \tilde{\lambda}_i \in \bracks{\lambda_i + \epsilon_n, \lambda_i + \epsilon_1}
    \end{equation}
\end{lemma}
Thus,  for the matrix $\bB$ we have that 
the condition number of the eigenvalues of  $\bB $ can be bounded as follows
\begin{align}
    \lambda_{\min}(\nabla f(\bx)) + \lambda_{\min}(\bE) \leq \lambda_{\min}(\bB) \leq \lambda_{\min}(\nabla f(\bx)) + \lambda_{\max}(\bE)\\
    \lambda_{\max}(\nabla f(\bx)) + \lambda_{\min}(\bE) \leq \lambda_{\max}(\bB) \leq \lambda_{\max}(\nabla f(\bx)) + \lambda_{\max}(\bE)
\end{align}
 For $\norm{\bE}_2 < \mu$ we have that 
 \begin{align}
    \dfrac{\lambda_{\max}(\nabla f(\bx)) + \lambda_{\min}(\bE)}{\lambda_{\min}(\nabla f(\bx)) + \lambda_{\max}(\bE)} &\leq \kappa(\bB) \leq \dfrac{\lambda_{\max}(\nabla f(\bx)) + \lambda_{\max}(\bE)}{\lambda_{\min}(\nabla f(\bx)) + \lambda_{\min}(\bE) }\\
 \dfrac{\kappa(\nabla f(\bx)) + \lambda_{\min}(\bE)/\lambda_{\min}(\nabla f(\bx))}{1 + \lambda_{\max}(\bE)/\lambda_{\min}(\nabla f(\bx))} &\leq \kappa(\bB) \leq \dfrac{\kappa(\nabla f(\bx)) + \lambda_{\max}(\bE)/\lambda_{\min}(\nabla f(\bx))}{1 + \lambda_{\min}(\bE)/\lambda_{\min}(\nabla f(\bx))}
 \end{align}
 And since $\twonorm{\bE}$ is small the two condition numbers are close.

\end{document}